\documentclass[times, review, 10pt]{elsarticle}
\usepackage[numbers]{natbib}
\usepackage{float}
\usepackage{graphicx}
\usepackage{subfigure}

\usepackage{diagbox} 
\usepackage{soul}

\PassOptionsToPackage{table}{xcolor}
\usepackage{xcolor}

\usepackage{float} 

\usepackage{amsthm}
\usepackage{amsmath}
\usepackage{amssymb}

\newtheorem{theorem}{Theorem}[section]

\newtheorem{remark}{Remark}
\usepackage{subfigure} 
\usepackage{enumitem}

\usepackage{comment}

\definecolor{teal}{rgb}{0.0, 0.5, 0.5}

 \usepackage{amsmath,amssymb,amsfonts}

\usepackage{hyperref}
\hypersetup{
    colorlinks=true,
    linkcolor=blue,
    filecolor=blue,      
    urlcolor=blue,
    citecolor=blue,
}

\usepackage{caption}

\usepackage{subfigure} 

\def\,{$\mskip\thinmuskip$} \def\!{$\mskip-\thinmuskip$}

\usepackage{algorithm}

\usepackage[algo2e, noend]{algorithm2e}
\usepackage{algorithmicx}
\usepackage{adjustbox}

\SetAlCapSkip{2pt}          
\SetAlgoInsideSkip{1pt}     

\def\BibTeX{{\rm B\kern-.05em{\sc i\kern-.025em b}\kern-.08em
    T\kern-.1667em\lower.7ex\hbox{E}\kern-.125emX}}

\begin{document}
\title{Spike Agreement Dependent Plasticity: A scalable Bio-Inspired learning paradigm for Spiking Neural Networks}

\author[1]{Saptarshi Bej}
\author[2]{Muhammed Sahad E}
\author[1]{Gouri Lakshmi}
\author[1]{Harshit Kumar}
\author[1]{Pritam Kar}
\author[2]{Bikas C Das}

\affiliation[1]{School of Data Science, Indian Institute of Science Education and Research, Thiruvananthapuram, India}
\affiliation[2]{School of Physics, Indian Institute of Science Education and Research, Thiruvananthapuram, India}

\begin{abstract}
We introduce Spike Agreement Dependent Plasticity (SADP), a biologically inspired synaptic learning rule for Spiking Neural Networks (SNNs) that relies on the agreement between pre- and post-synaptic spike trains rather than precise spike-pair timing. SADP generalizes classical Spike-Timing-Dependent Plasticity (STDP) by replacing pairwise temporal updates with population-level correlation metrics such as Cohen’s $\kappa$. The SADP update rule admits linear-time complexity and supports efficient hardware implementation via bitwise logic. Empirical results on MNIST and Fashion-MNIST show that SADP, especially when equipped with spline-based kernels derived from our experimental iontronic organic memtransistor device data, outperforms classical STDP in both accuracy and runtime. Our framework bridges the gap between biological plausibility and computational scalability, offering a viable learning mechanism for neuromorphic systems.
\end{abstract}

\maketitle


\section{Introduction}\label{sec:introduction}
Spike-Timing-Dependent Plasticity (STDP) has long stood as a foundational model for synaptic learning, rooted in the temporally asymmetric Hebbian principle that causally related spikes reinforce synaptic strength \cite{bi1998synaptic,markram1997regulation}. When a presynaptic neuron fires shortly before a postsynaptic neuron, Long-Term Potentiation (LTP) occurs; if the order is reversed, Long-Term Depression (LTD) is induced. This pairwise mechanism is supported by extensive experimental evidence and has been successfully used to model learning in Spiking Neural Networks (SNNs) \cite{bi2001synaptic,sjostrom2001rate}. Unsupervised STDP-based learning has also proven effective in extracting visual features from natural scenes using biologically plausible SNNs~\cite{MasquelierThorpe2007,Ferre2018,Kheradpisheh2016}. These models, when combined with competitive inhibition or convolutional layers, have shown strong performance on classification benchmarks such as MNIST, CIFAR-10, and ETH-80.

Despite its biological plausibility, classical STDP faces two major limitations in computational settings. First, it relies on millisecond-scale spike timing precision, which is difficult to achieve and maintain in noisy environments such as the cortex \cite{Markram2011,Markram2012}. Second, its pairwise nature leads to quadratic computational complexity, as it requires evaluation of all spike pairings between pre- and post-synaptic neurons. These constraints hinder scalability and robustness, particularly in deep or real-time neuromorphic systems \cite{Tian2025}.

To overcome these limitations, recent research has shifted toward incorporating higher-order and population-level features into synaptic plasticity rules. Notably, empirical findings demonstrate that coordinated spiking activity—or neuronal synchrony—plays a critical role in effective synaptic modification, even in the absence of direct spike-pair causality \cite{Tian2025,Subthreshold2025}. This has led to the development of Spike-Synchrony-Dependent Plasticity (SSDP), which enhances STDP by modulating weight updates based on the degree of synchrony among neuron groups \cite{Tian2025}. SSDP improves learning stability and better reflects biological plasticity by capturing temporally proximate activity within spiking ensembles.

Building upon these advances, we propose \textbf{Spike Agreement-Dependent Plasticity (SADP)}, a novel learning rule that reinterprets synaptic plasticity as a function of spike train agreement rather than individual spike order. Unlike classical STDP, SADP does not require strict causality. Instead, it quantifies the alignment between pre- and post-synaptic spike trains over extended windows using statistical metrics such as Cohen’s $\kappa$, offering a causal-agnostic but correlation-sensitive framework. 

SADP further reduces computational burden by eliminating pairwise spike comparisons, resulting in linear-time updates that are well suited for hardware implementation. From a biological standpoint, SADP aligns with emerging views of plasticity as a multi-factorial process, influenced not only by spike timing but also by membrane dynamics, neuromodulators, and population activity patterns \cite{foncelle2018Modulation,clopath2010connectivity,pfister2006triplets}. Its compatibility with diverse plasticity profiles makes it a compelling candidate for scalable, interpretable, and efficient learning in neuromorphic systems.

SADP represents a conceptual and practical evolution beyond STDP and SSDP: it unifies spike-based and population-aware learning in a way that is both biologically grounded and computationally viable, positioning it as a foundation for the next generation of SNNs operating in real-world environments.

Our primary contribution is conceptual: we reframe synaptic learning as a function of spike train alignment, enabling robust, biologically plausible learning that is well-suited for noisy, population-driven neural dynamics. Unlike prior work that couples novel learning rules with complex network architectures or task-specific modifications, our approach is architecture-agnostic and compatible with standard spiking neuron models. Thus, this work contributes both a new perspective on spike-driven learning and a mathematically rigorous foundation for its stability and applicability across neuromorphic and biological contexts.
In Section \ref{sec:related} we detail a bit more on the history of STDP and recent directions of research towards Spike-Synchrony-Dependent Plasticity (SSDP). Thereafter, we detail on our novel approach SADP in Section \ref{sec:SADP}.

\paragraph{\textbf{Contributions}} The key contributions of this work are summarized below:

\begin{enumerate}[topsep=0pt,itemsep=0pt,parsep=0pt]
    \item We propose a novel unsupervised learning rule SADP, that generalizes classical STDP by modulating weight updates based on the statistical agreement between pre- and post-synaptic spike trains.
    
    \item We introduce a flexible kernel-based formulation for SADP, allowing device-specific weight update curves to be implemented efficiently via spline interpolation, making the rule highly compatible with emerging memtransistor-based neuromorphic hardware.
    
    \item We empirically validate SADP on unsupervised image classification tasks (MNIST and Fashion-MNIST), demonstrating improved performance and robustness over classical STDP across multiple encoding schemes and network configurations.
\end{enumerate}

\section{Related research} \label{sec:related}

\subsection{Background of Hebbian Learning}
\label{sec:hebbian_comparison}

The Hebbian learning rule, often summarized as "cells that fire together wire together" \cite{hebb1949organization}, updates synaptic weights based on the coincident activity of pre- and post-synaptic neurons. In its classical form, Hebbian plasticity updates the synaptic weight $w_{ij}$ as:
\[
\Delta w_{ij} = \eta \cdot x_i \cdot y_j,
\]
where $x_i$ and $y_j$ represent the spike activities (or rates) of the pre- and post-synaptic neurons, respectively, and $\eta$ is a learning rate. When applied to temporally extended spike trains, this corresponds to a dot product of binary activity vectors over time:
\begin{equation}\label{hebbain_learning}
    \Delta w_{ij} = \eta \sum_{t=1}^T x_i(t) y_j(t)
\end{equation}

which captures raw co-activation but not the precise timing or structure of spike trains.

Several extensions of Hebbian learning have been proposed, including:
\begin{enumerate}[topsep=0pt,itemsep=0pt,parsep=0pt]
    \item \textbf{Oja’s rule} \cite{oja1982simplified}: A normalized variant to prevent runaway growth of weights.
    \item \textbf{BCM theory} \cite{bienenstock1982theory}: Introduces a sliding threshold for long-term potentiation (LTP) vs depression.
    \item \textbf{Rate-based Hebbian rules} \cite{gerstner2014neuronal}: Widely used in continuous neural models but often lacking biological spike timing fidelity.
\end{enumerate}

Despite their simplicity and efficiency (linear-time updates, no spike-pair matching), Hebbian rules have several limitations in the context of neuromorphic and event-driven systems. However, most variants of Hebbian learning are designed for digital rate-based models and lack compatibility with device-level synaptic dynamics.

\subsection{Background of STDP}

Spike-Timing-Dependent Plasticity (STDP) represents a fundamental mechanism of synaptic modification in the brain, characterized as a temporally asymmetric form of Hebbian learning. This process precisely adjusts the strength of connections between neurons based on the relative timing of their electrical impulses, known as action potentials or spikes~\cite{bi1998synaptic,markram1997regulation}. This intricate timing-dependent modulation is widely regarded as a cornerstone for how the brain acquires and retains information, as well as how its complex neuronal circuits undergo development and refinement throughout an organism's lifespan~\cite{bi2001synaptic}.

In the canonical formulation, the update kernel depends on the spike timing difference $\Delta t = t_{\text{post}} - t_{\text{pre}}$:

\begin{equation}\label{stdp_learning} 
K_{\text{STDP}}(\Delta t) =
\begin{cases}
A_{+} e^{-\Delta t / \tau_{+}}, & \Delta t > 0, \\
- A_{-} e^{\Delta t / \tau_{-}}, & \Delta t < 0,
\end{cases}
\end{equation}

where $A_{+}, A_{-} > 0$ are the maximal potentiation and depression amplitudes, and $\tau_{+}, \tau_{-}$ are decay time constants. Small positive $\Delta t$ (presynaptic before postsynaptic) is interpreted as causal, producing strong potentiation; small negative $\Delta t$ (postsynaptic before presynaptic) is interpreted as anti-causal, producing strong depression. This yields an asymmetric, causality-driven plasticity window.

The core principle of STDP dictates that if a presynaptic neuron fires its spike a few milliseconds before a postsynaptic neuron generates its own action potential, the synaptic connection between them typically undergoes Long-Term Potentiation (LTP), signifying a strengthening of that connection. Conversely, if the presynaptic spike arrives after the postsynaptic spike, the same synapse experiences Long-Term Depression (LTD), leading to its weakening~\cite{bi1998synaptic}. This specific temporal window, often termed the STDP function or learning window, is not universal but can exhibit variations across different types of synapses, reflecting the diverse computational roles of various brain regions. The precise detection of these millisecond-scale spike timings is often mediated by NMDA receptors, which function as critical coincidence detectors within the synaptic machinery~\cite{sjostrom2001rate}.

In computational models, STDP has been employed for unsupervised learning of visual and auditory patterns. For instance, Masquelier and Thorpe~\cite{MasquelierThorpe2007} demonstrated that spike-timing-based plasticity could be used to extract visual features from natural image sequences. Extensions of these methods to convolutional spiking networks~\cite{Kheradpisheh2016,Kheradpisheh2018} and auditory pattern recognition~\cite{Dong2018} have further validated STDP’s potential in pattern recognition domains.

The appeal of STDP in computational neuroscience stems from its strong biological grounding and its ability to enable synaptic learning at low firing rates, by exploiting precise temporal correlations in spike timing rather than relying on high activity levels~\cite{song2000competitive,bi2001synaptic}. STDP is frequently described as ``temporally causal'' and ``temporally asymmetric'' because it assigns synaptic changes based on the temporal order of pre- and postsynaptic spikes~\cite{Markram2012}. This causal framing extends the classic Hebbian postulate by emphasizing the directionality of influence, suggesting that biological synaptic plasticity mechanisms may encode not just co-occurrence but also the temporal precedence of neural events. For SNNs, this insight supports the design of learning algorithms that are more interpretable and biologically faithful by integrating temporally causal learning rules.

The experimental foundations of STDP were laid by Markram and Sakmann~\cite{markram1997regulation} and Bi and Poo~\cite{bi1998synaptic,bi2001synaptic}, who demonstrated the temporally asymmetric and bidirectional nature of synaptic plasticity within temporal windows of 10 to 100 milliseconds. Early theoretical models focused on firing-rate-based plasticity~\cite{song2000competitive}, but later shifted toward temporal spike correlations, prompting the development of recurrent neural models using causal STDP learning rules~\cite{yang2025causal}. These models have been effective at supporting the emergence of stable neuronal assemblies and increasing information capacity.

Recent literature acknowledges that connectivity among neurons is not solely governed by spike timing. Additional biophysical factors, such as postsynaptic membrane potential, calcium concentration, and network-wide activity, significantly modulate plasticity. Studies have highlighted the limitations of canonical STDP and emphasized the need for multifactor plasticity models~\cite{pfister2006triplets,clopath2010connectivity,foncelle2018Modulation}. These models present a more biologically accurate picture of learning and enable the design of SNNs that better reflect biological computation.

\subsection{Limitations of Classical STDP in Computational Models}

Despite its biological appeal, classical STDP faces critical limitations in computational applications, particularly in large-scale, noisy environments. STDP relies heavily on precise spike timing within millisecond windows to differentiate between long-term potentiation (LTP) and long-term depression (LTD) \cite{Markram2012}. However, the experimental spike timing curves often derive from noisy biological data, making computational implementations of STDP vulnerable to fluctuations in spike timings and environmental noise \cite{Markram2011}. In dynamic systems, such precision demands can lead to instability and poor generalization, especially in high-frequency or stochastic spiking environments.

The event-driven and pairwise nature of STDP imposes a high computational cost. Tracking precise spike timings and updating weights asynchronously across millions of synapses is significantly more intensive than conventional gradient-based methods. Furthermore, STDP’s local, pairwise updates neglect higher-order neural interactions, which limits its ability to model the ensemble dynamics essential for cognitive functions \cite{Tian2025}. Integrating STDP with backpropagation remains difficult due to mismatched update directions, leading to subpar classification performance in deep spiking networks.

\subsection{Influence of Neural Synchrony and Collective Dynamics on Plasticity}

Classical STDP models based solely on dyadic spike interactions face increasing scrutiny in light of growing biological evidence highlighting the importance of neuronal synchrony---the coordinated firing of groups of neurons---as a key driver of synaptic plasticity and robust learning \cite{Markram2011,Tian2025}. Synchrony serves as a fundamental coding mechanism that enhances the reliability of neural communication and supports long-range feature binding across circuits \cite{Tian2025}.

Experimental studies have shown that synchronously firing populations can induce long-term potentiation (LTP) even in the absence of pairwise spike causality, challenging the traditional STDP framework \cite{Markram2011}. Synchronization is also linked to increased plasticity thresholds and improved attractor dynamics, contributing to more stable learning processes \cite{Tian2025}. Moreover, the level of synchrony in a population is tightly coupled to its mean firing rate, suggesting that rate coding and synchrony coding are interdependent perspectives on the same underlying network dynamic \cite{Chawla1999}.

This challenges the prevalent asynchronous state hypothesis, favoring instead a view where weak but consistent synchrony plays a central role in shaping cortical variability and computation \cite{Subthreshold2025}. As such, the limitations of purely pairwise STDP motivate a shift toward more global models of plasticity.

Beyond pairwise spike timing, STDP-based mechanisms have also been shown to support probabilistic inference in cortical microcircuits~\cite{Nessler2013,Habenschuss2013}. These works demonstrate that population codes and WTA dynamics can emerge naturally in networks trained with STDP alone, broadening its relevance beyond low-level plasticity.

\subsection{Emergence of Spike-Synchrony-Dependent Plasticity (SSDP)}

To incorporate these insights, Spike-Synchrony-Dependent Plasticity (SSDP) extends the STDP rule by integrating not only spike timing but also the degree of synchrony among neuron groups as a modulatory signal for weight updates \cite{Tian2025}. This enables more scalable and noise-robust learning by capturing temporally proximate spikes rather than requiring strictly causal interactions.

SSDP has been shown to improve learning stability and classification performance across various architectures, including spiking ResNets and transformer-based SNNs, all while operating in a fully event-driven and energy-efficient manner \cite{Tian2025}. The emerging consensus that rate and synchrony coding are complementary rather than mutually exclusive \cite{Chawla1999} provides a strong theoretical foundation for SSDP as a unifying mechanism for biologically plausible and hardware-friendly learning in SNNs.

\begin{remark}
The limitations of classical STDP—its strict dependence on precise spike timing, high computational cost due to pairwise comparisons, and vulnerability to noise—combined with the growing evidence for population-level synchrony as a driver of plasticity, motivate the need for more global, noise-tolerant, and hardware-efficient learning rules. While SSDP partially addresses this by incorporating synchrony into STDP, it still retains causal timing dependencies and pairwise update complexity. The proposed SADP learning paradigm described in the next Section bridges this gap by replacing millisecond-scale timing constraints with a global spike-train agreement metric (Cohen's $\kappa$), thereby unifying temporal sensitivity, statistical robustness, and linear-time complexity. This design makes SADP inherently suited for low-latency, device-calibrated, and scalable neuromorphic learning.
\end{remark}

\section{Spike Agreement-Dependent Plasticity (SADP)} \label{sec:SADP}

We shall now describe the novel SADP learning paradigm on a simple SNN built on leaky integrate-and-fire (LIF) dynamics. The network processes inputs encoded as coded (rate coding or other alternatives such as burst coding) spike trains over a fixed number of discrete time steps $T$. Each input sample is represented as a 3D tensor $\mathbf{X} \in \mathbf{R}^{B \times N_{\text{in}} \times T},$ where $B$ is the batch size, $N_{\text{in}}$ is the number of input neurons (e.g., 784 for flattened MNIST), $T$ is the number of discrete time steps (e.g., 10 or 30, depending on encoding resolution). The network consists of one or more fully connected spiking layers. Each layer integrates the spiking input current over time and emits binary spike trains as output. The membrane potential of each neuron in the layer follows leaky integration, reset, and thresholding dynamics. The output of the network is a spatio-temporal spike train tensor $\mathbf{S} \in \mathbf{R}^{B \times N_{\text{out}} \times T},$ where $N_{\text{out}}$ is the number of output neurons. These output spike trains can be aggregated over time to infer features learnt using SADP and can be used by downstream workflows for further computation. 

Each layer is parameterized by a trainable synaptic weight matrix $\mathbf{W} \in \mathbf{R}^{N_{\text{in}} \times N_{\text{out}}},$ where $w_{ij} \in \mathbf{W}$ denotes the strength of the synapse from pre-synaptic neuron $i$ to post-synaptic neuron $j$. This matrix transforms pre-synaptic spike input at each time step into post-synaptic current. Synaptic weights are initialized from a Rademacher distribution $w_{ij} \sim \text{Uniform}\{-1, +1\},$ which promotes early activity symmetry and propagation. Training is conducted to learn $\mathbf{w}$ in an unsupervised manner using the SADP rule, where synaptic updates are driven by spike train agreement rather than precise timing.

\subsection{Neuron Dynamics in SADP-adapted SNN}

The SNN integrates LIF dynamics with the proposed SADP rule. Each output neuron's membrane potential evolves over discrete time steps $t = 1, \dots, T$ based on a decay factor $\lambda \in (0, 1)$ and the incoming synaptic current:

\begin{equation}
\mathbf{V}_t = \lambda \mathbf{V}_{t-1} + \mathbf{I}_t, \quad \text{where} \quad \mathbf{I}_t = \mathbf{X}_t \mathbf{W}.
\end{equation}

Here, $\mathbf{X}_t \in \mathbf{R}^{B \times N_{\text{in}}}$ denotes the input spikes at time $t$ for a batch of $B$ samples and $N_{\text{in}}$ input neurons. Also, $\mathbf{V}_{t-1} \in \mathbf{R}^{B \times N_{\text{out}}}$. The weight matrix $\mathbf{W} \in \mathbf{R}^{N_{\text{in}} \times N_{\text{out}}}$ projects inputs to $N_{\text{out}}$ output neurons.

To ensure comparability of membrane potentials across neurons and time steps, the potentials are normalized at each time step:

\begin{equation}
\tilde{\mathbf{V}}_t = \frac{\mathbf{V}_t - \min(\mathbf{V}_t)}{\max(\mathbf{V}_t) - \min(\mathbf{V}_t) + \epsilon},
\end{equation}

where $\epsilon > 0$ is a small constant to avoid division by zero.

Spikes are emitted using a hard thresholding mechanism:

\begin{equation}
\mathbf{S}_t = H(\tilde{\mathbf{V}}_t - \theta),
\end{equation}

where $H(\cdot)$ is the Heaviside step function applied element-wise, and $\theta \in (0, 1)$ is the spiking threshold. When a spike is generated, the membrane potential is reset to prevent immediate reactivation:

\begin{equation}
\mathbf{V}_t \leftarrow \mathbf{V}_t \cdot (1 - \mathbf{S}_t).
\end{equation}

Across all time steps, the full spike train for the batch is recorded as:

\begin{equation}
\mathbf{S} = \{ \mathbf{S}_t \}_{t=1}^T \in \mathbf{R}^{B \times N_{\text{out}} \times T}.
\end{equation}

\subsection{SADP learning rule}

Unlike classical STDP, which relies on pairwise spike timing differences, SADP defines plasticity based on spike-train agreement using Cohen’s $\kappa$ coefficient. For each input sample in a batch, let the pre-synaptic spike train over $T$ steps from neuron $i$ be denoted as $\mathbf{x}_{b,i,:} \in \{0,1\}^T$, and the corresponding post-synaptic output from neuron $j$ as $\mathbf{S}_{b,j,:} \in \{0,1\}^T$, where $b \in \{1, \dots, B\}$ indexes the batch.

We define the SADP agreement score for each synapse $(i,j)$ and sample $b$ as:
\begin{equation}
\kappa_{ij}^{(b)} = \frac{p_0^{(b)} - p_e^{(b)}}{\max(1 - p_e^{(b)}, \epsilon)},
\end{equation}
where $\epsilon > 0$ ensures numerical stability, and
\begin{align}
p_0^{(b)} &= \frac{1}{T} \sum_{t=1}^T \mathbf{1} \left( \mathbf{X}_{b,i,t} = \mathbf{S}_{b,j,t} \right), \\
p_e^{(b)} &= \left( \frac{1}{T} \sum_{t=1}^T \mathbf{X}_{b,i,t} \right) \left( \frac{1}{T} \sum_{t=1}^T \mathbf{S}_{b,j,t} \right) + \left( \frac{1}{T} \sum_{t=1}^T (1 - \mathbf{X}_{b,i,t}) \right) \left( \frac{1}{T} \sum_{t=1}^T (1 - \mathbf{S}_{b,j,t}) \right).
\end{align}

The update signal is computed by applying a bounded learning function $\mathcal{L}: [-1,1] \to \mathbb{R}$ to each $\kappa_{ij}^{(b)}$, and then averaged over the batch:
\begin{equation}
\Delta w_{ij}^{(t)} = \frac{\eta_t}{B} \sum_{b=1}^B \mathcal{L} \left( \kappa_{ij}^{(b)} \right).
\end{equation}

The synaptic weights are updated as:
\begin{equation}
w_{ij}^{(t+1)} = \operatorname{clip} \left( \operatorname{sign}\left(w_{ij}^{(t)} + \Delta w_{ij}^{(t)}\right) \cdot \max \left( \left|w_{ij}^{(t)} + \Delta w_{ij}^{(t)} \right|, \epsilon \right), -1, 1 \right),
\end{equation}
where $\operatorname{clip}(x, -1, 1)$ ensures the weights remain within bounds, and the $\epsilon$-floor prevents synaptic silencing.

\begin{figure}[ht]
    \centering
    \includegraphics[width=0.95\textwidth]{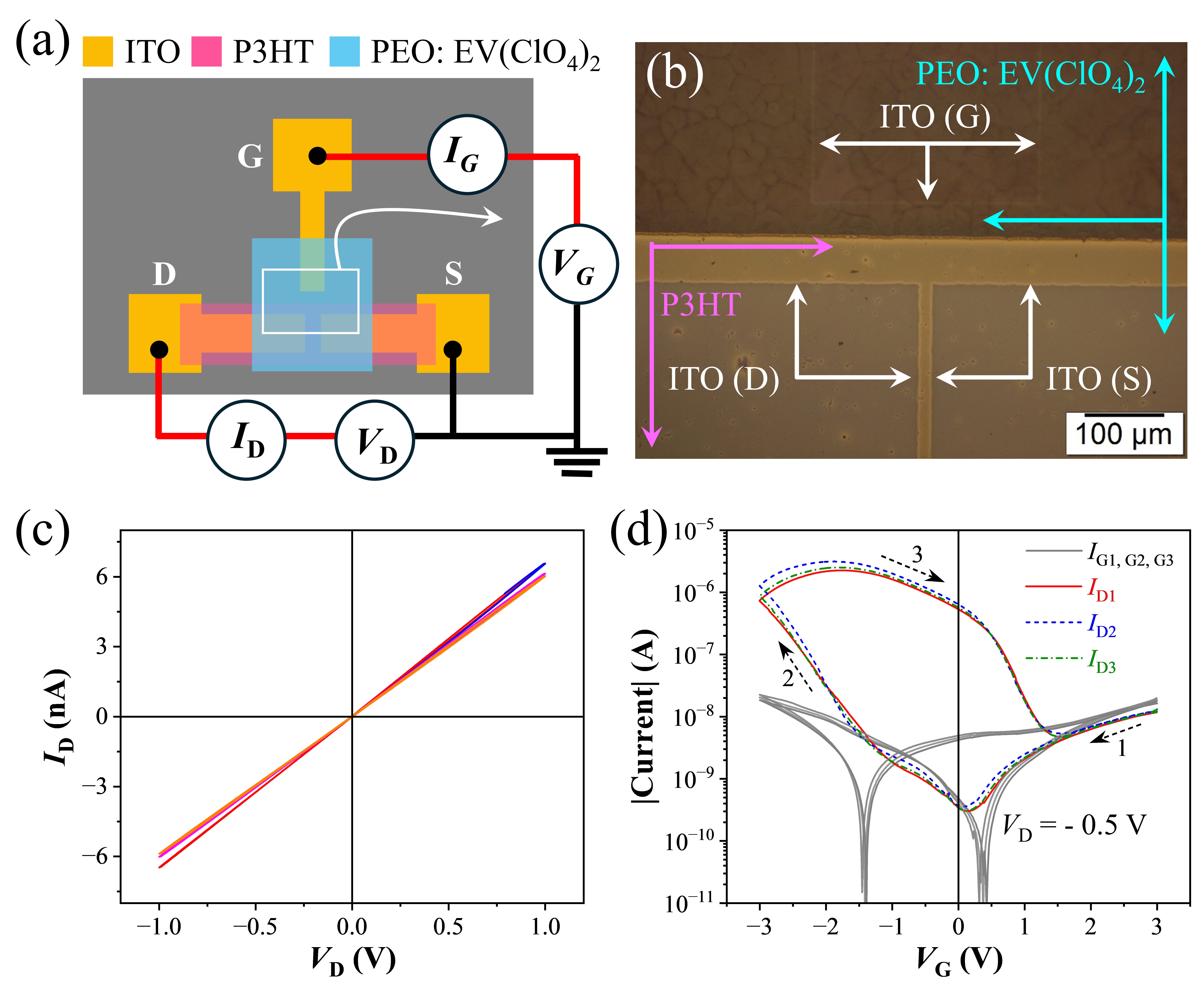}
    \caption{
        Device fabrication and electrical characterizations. 
        (a) Schematic representation of the side-gated iontronic organic memtransistor with electrical connections. 
        (b) Optical microscope image over the channel as marked by a white lined rectangle in (a) depicting transparent ITO electrodes (source, drain, and gate), P3HT channel, and solid redox electrolyte covered entire area, which are labeled accordingly. 
        (c) Current-voltage ($I_D$-$V_D$) characteristic recorded between the S-D with the photo-lithographically patterned P3HT channel only. 
        (d) Transfer ($I_D$-$V_G$) characteristics for three consecutive gate voltage sweeps in loop between $\pm3.0$ V at an applied $V_D = ‒0.5$ V. It is overlaid with the simultaneously recorded gate currents ($I_G$-$V_G$). The arrows indicate the $V_G$ sweep directions.
    }
    \label{fig:oect_characterization}
\end{figure}

\begin{remark}[From STDP to SADP: Device-Grounded Learning]
Spike-based plasticity in neuromorphic systems is fundamentally shaped by the physical characteristics of their synaptic elements, such as iontronic memtransistor or phase-change devices. These components often exhibit nonlinear and asymmetric conductance changes in response to sequences of potentiation and depression pulses, and these device-specific responses may deviate significantly from the smooth, exponential curves of idealized biological STDP.

\paragraph{\textbf{SADP reinterpretation}}
Spike-Agreement-Dependent Plasticity (SADP) replaces the dependence on spike-pair causality with a dependence on \emph{global spike-train agreement}, quantified by an agreement index $\kappa_{ij} \in [-1,1]$. This value is mapped to a synthetic temporal coordinate $\delta \in [-1,1]$, where:
\[
\delta_{\mathrm{pot}} \in (0,1] \quad\text{represents high spike-train agreement,}
\]
\[
\delta_{\mathrm{dep}} \in [-1,0) \quad\text{represents strong anti-correlation.}
\]
Potentiation is strongest for large positive $\delta$; depression is strongest for large negative $\delta$. Conceptually, the positive and negative halves of the STDP window are swapped to align with agreement rather than causality, effectively mirroring each half of the STDP kernel across $\delta = 0$. 

\paragraph{\textbf{Device-based kernel construction in SADP}}
When experimental device data are available, we record the conductance $G(t)$ over sequences of potentiation and depression pulses, and compute the normalized update:
\[
\frac{\Delta G}{G_0}(t) = \frac{G(t+1) - G(t)}{G(t) + \varepsilon},
\]
where $\varepsilon > 0$ ensures numerical stability. These updates are assigned to $\delta_{\mathrm{pot}}$ or $\delta_{\mathrm{dep}}$ according to the intended $\kappa$ mapping. We then fit \emph{smoothing splines} to the sorted data:
\[
f_{+}(\delta) = \mathrm{UnivariateSpline}(\delta_{\mathrm{pot}}, \Delta g_{\mathrm{pot}}, s=0.1),
\]
\[
f_{-}(\delta) = \mathrm{UnivariateSpline}(\delta_{\mathrm{dep}}, \Delta g_{\mathrm{dep}}, s=0.01),
\]
yielding continuous, differentiable potentiation and depression kernels. The smoothing parameter $s$ controls the trade-off between fidelity to measured data and smoothness of the fit.

\paragraph{\textbf{Synthetic kernel construction in SADP (no device data)}}
If no device measurements are available, one can start from a fitted parametric STDP kernel $K_{\mathrm{STDP}}(\Delta t)$ and generate the SADP kernel by exchanging the positive and negative $\Delta t$ segments:
\[
K_{\mathrm{SADP}}(\delta) =
\begin{cases}
K_{\mathrm{STDP}}(\delta - 1), & \delta > 0, \\
K_{\mathrm{STDP}}(\delta + 1), & \delta < 0.
\end{cases}
\]
where $K_\mathrm{STDP}$ follows from Equation \ref{stdp_learning}. This preserves the magnitude profile while swapping the causal and anti-causal roles, aligning the update rule with global spike-train agreement. As illustrated in Fig.~\ref{fig:stdp_kernel_ideal} and Fig.~\ref{fig:sadp_kernel_ideal}, 
the ideal reference kernels provide smooth, symmetric potentiation and depression 
branches that serve as analytical baselines for SADP learning.
\end{remark}

\subsection*{Neuromorphic Iontronic Memtransistor Platform}
The SADP learning framework was experimentally realized and systematically calibrated using the recorded data from a side-gated iontronic organic memtransistor, fabricated with high robustness through a conventional photolithography process. The device architecture incorporates lithographically patterned ITO source (S), drain (D), and gate (G) electrodes, bridged by a P3HT polymer channel electrically isolated from the gate. A drop-casted PEO:EV(ClO$_4$)$_2$ solid redox-electrolyte layer was deposited as the gate dielectric, as illustrated in the schematic of Figure \ref{fig:oect_characterization}a. Complementarily, Figure \ref{fig:oect_characterization}b presents an optical micrograph of the active channel region, clearly delineating all structural counterparts with high precision. Detailed fabrication procedures and comprehensive electrical characterizations are provided in Appendix B.

This architecture intrinsically offers low leakage current and significant channel conductance tunability, primarily governed by reversible redox interactions and the efficient migration of counter-ions across the channel–electrolyte interface under applied gate bias polarity \cite{sagar2022emulation}. Figure \ref{fig:oect_characterization}(c) displays the hysteresis-free current–voltage ($I_D$-$V_D$) characteristics of the patterned P3HT channel (prior to electrolyte deposition), confirming its robust semiconducting nature. Following electrolyte integration, facile and reproducible memtransistor behavior was consistently observed, as reflected in the transfer characteristics shown in Figure \ref{fig:oect_characterization}f. The pronounced hysteresis in drain current ($I_D$) across opposite bias sweeps is emblematic of enhancement-mode p-channel memtransistor behavior \cite{lee2013effective, mukherjee2021superionic}. Furthermore, the significantly suppressed gate current recorded in Figure \ref{fig:oect_characterization}ad originates from the redox coupling between DPP-DTT and EV(ClO$_4$)$_2$ under gate bias \cite{sagar2019unconventional}.

Organic iontronic memtransistors have already emerged as strong contenders for next-generation neuromorphic hardware. Particularly, the synaptic behavior was demonstrated through linear and symmetric potentiation/depression (P/D) responses to the input voltage pulse trains, enabling direct mapping of spline-calibrated SADP kernels to device conductance updates \cite{vandeBurgt2018organic}. The P/D characteristics, recorded as post-synaptic responses under gate-applied voltage pulse sequences (Figure \ref{fig:memtransistor_data} (left)), are illustrated for our side-gated iontronic memtransistor in Figure \ref{fig:memtransistor_data} (right). These measured P/D curves were subsequently employed to derive smooth, bounded spline kernels that parameterize the SADP learning rule. This device-calibrated framework ensures that simulated synaptic updates closely reproduce the experimentally observed conductance dynamics, thereby enabling a physically grounded evaluation of SADP within neuromorphic hardware contexts. Figures \ref{fig:stdp_kernel} and \ref{fig:sadp_kernel} further illustrate a device-specific STDP kernel extracted from the experimental P/D curve and the corresponding SADP kernel utilized in this study.

\begin{figure}[htbp!]
    \centering
    
    \subfigure[Synaptic Functionality and ANN simulations. (left) The pre-synaptic input pulse scheme at the gate for 1000 consecutive excitatory write (W) (‒ 3.0 V, 50 ms) and inhibitory (E) (+ 1.0 V, 50 ms) voltage pulses. It include a read voltage pulse of ‒ 0.5 V, 50 ms after each W/E pulses. (right) Recorded post-synaptic P/D response curve for the input pulse trains in the diagram to the (left).]{
        \includegraphics[width=0.8\textwidth]{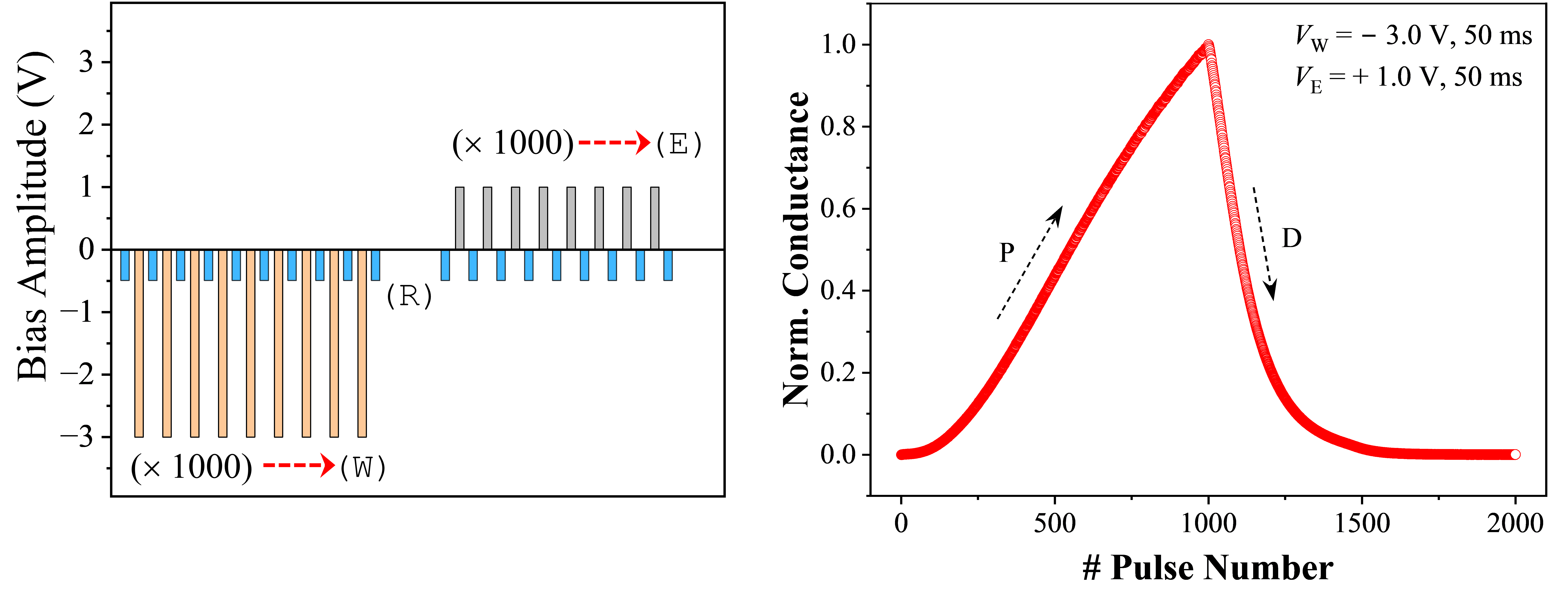}
        \label{fig:memtransistor_data}
    }
    
    \vspace{0.2cm} 

    \subfigure[Memtransistor-based STDP spline kernel showing potentiation and depression data with smoothing spline fits.]{
        \includegraphics[width=0.48\textwidth]{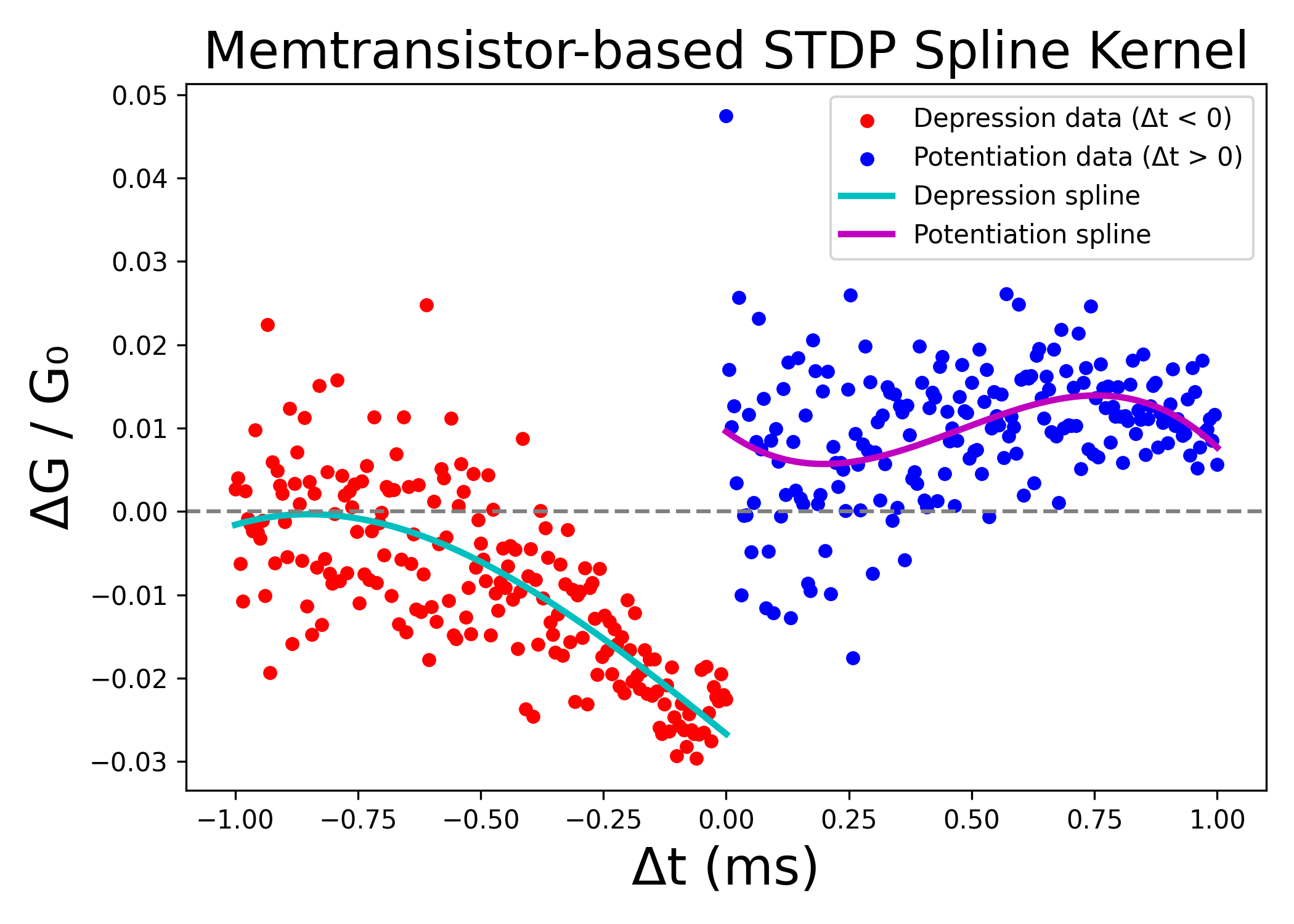}
        \label{fig:stdp_kernel}
    }
    \hfill
    \subfigure[Memtransistor-based SADP spline kernel obtained by remapping the STDP kernel to agreement coordinates, with potentiation and depression spline fits.]{
        \includegraphics[width=0.48\textwidth]{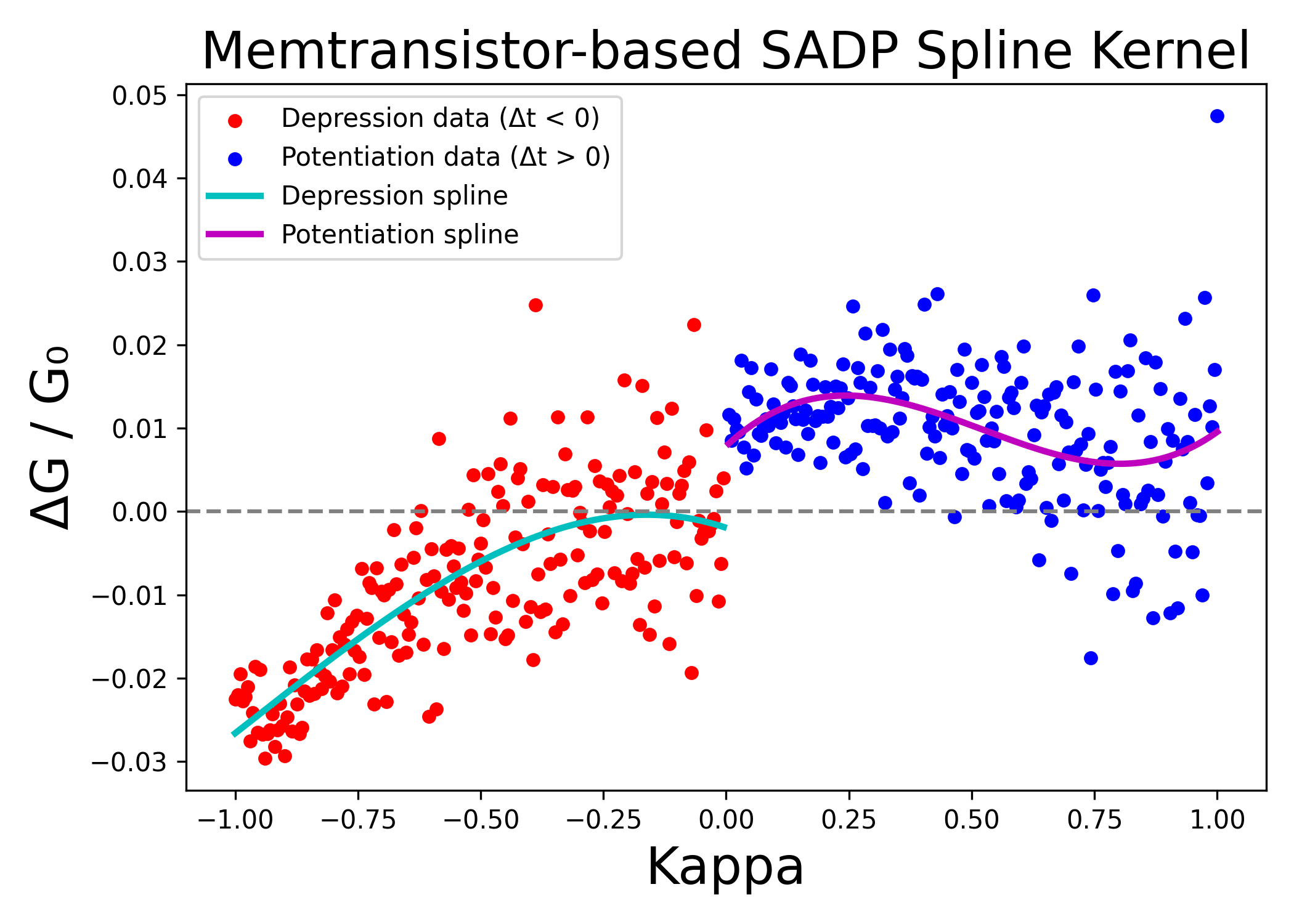}
        \label{fig:sadp_kernel}
    }

    \vspace{0.2cm} 

    \subfigure[Ideal STDP spline kernel in $\Delta t$ domain, showing exponential long-term potentiation (LTP) and depression (LTD) components.]{
        \includegraphics[width=0.48\textwidth]{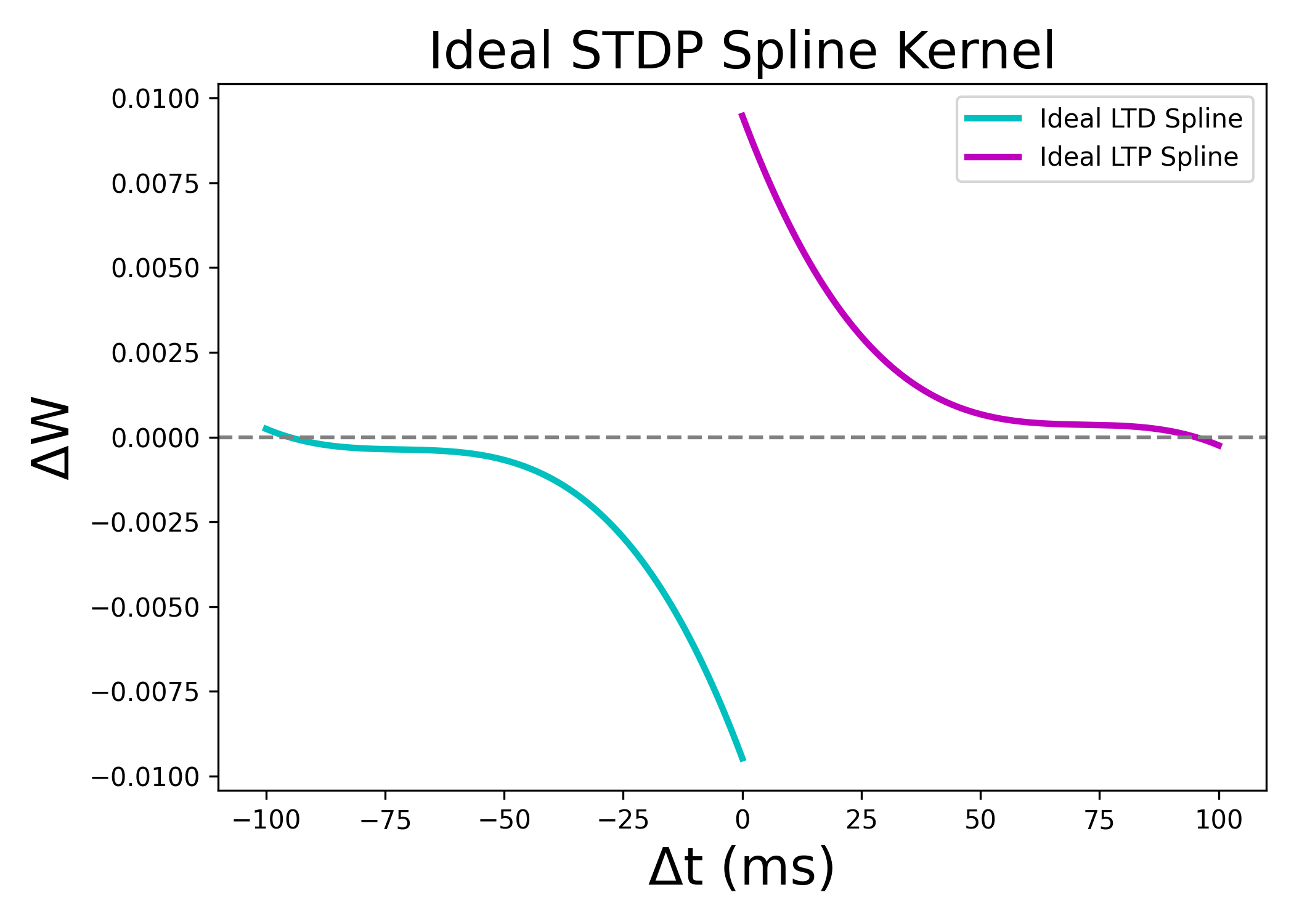}
        \label{fig:stdp_kernel_ideal}
    }
    \hfill
    \subfigure[Ideal SADP spline kernel in agreement ($\kappa$) domain, showing smooth symmetric potentiation and depression branches.]{
        \includegraphics[width=0.48\textwidth]{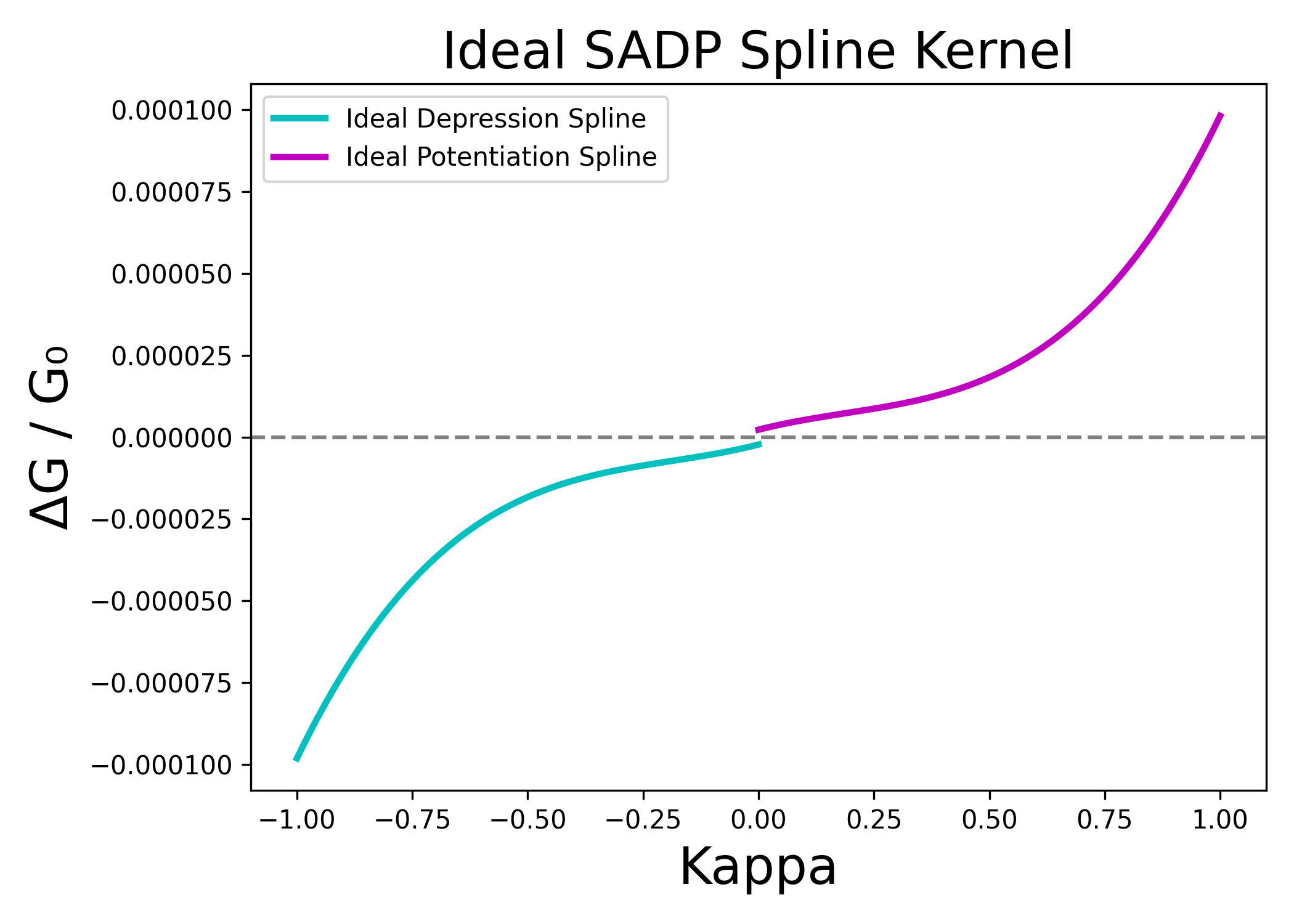}
        \label{fig:sadp_kernel_ideal}
    }

    \caption{\textbf{Memtransistor device characterization and learning kernels.} 
    (a) Electrical stimulation protocol and conductance modulation over write/erase cycles. 
    (b) STDP-based spline kernel derived from device measurements. 
    (c) SADP-based spline kernel obtained by remapping the STDP kernel to agreement coordinates. 
    (d) Ideal reference STDP kernel in $\Delta t$ domain. 
    (e) Ideal reference SADP kernel in $\kappa$ domain. 
    Together, these illustrate both device-calibrated and idealized kernel families for synaptic plasticity modeling.}
    \label{fig:memtransistor_kernels}
\end{figure}

\section{Benchmarking Experiments}
\label{sec:benchmarking}

To evaluate the performance and efficiency of our spike-based learning rules, we conducted a comprehensive benchmarking experiment involving multiple datasets, temporal encoding schemes, and network architectures. The primary aim of these experiments was to assess the trade-off between computational cost and representational power in \textbf{biologically inspired synaptic update mechanisms}, specifically comparing our proposed SADP variants (Spline (Device Specific and Analytical) and Linear) with classical spike-timing-dependent plasticity (STDP) and Hebbian learning rules. For all configurations, we measured classification accuracy, macro-averaged F1 score, and total end-to-end runtime, encompassing both unsupervised feature learning and downstream classifier training.

\textbf{Dataset and Encoding.} All models were evaluated on the MNIST and Fashion-MNIST (FMNIST) datasets, two standard pattern recognition datasets that are widely used to test Bio-Inspired pattern recognition models. The input images were first normalized to the $[0,1]$ range. Each image was then encoded into a temporal spike train using two biologically plausible strategies. In \textit{rate coding}, spikes were generated independently at each time step with probability proportional to pixel intensity, such that higher intensity pixels fired more frequently over a 10-step time window. In contrast, \textit{time-to-first-spike} (TTFS) coding encoded each pixel as a single spike at a latency inversely related to intensity; brighter pixels fired earlier, while darker pixels fired later or not at all. Label spike trains were generated to mirror the input encoding: multi-spike outputs for rate coding and single earliest spikes for TTFS.

\textbf{SADP Variants used for benchmarking.} 
Two spike-based learning rules were implemented for synaptic updates. 
The \textbf{Linear SADP} rule computed weight changes as a linear function of the $\kappa$-agreement score, which quantifies temporal similarity between pre- and post-synaptic spike trains. The update was scaled by fixed learning rates for potentiation and depression, applied over a 10-step window. 
The \textbf{Spline SADP} variant extended this approach by using spline-fitted synaptic dynamics derived from experimental iontronic memtransistor measurements. These splines, fitted to normalized conductance change curves ($\Delta G / G_0$), were used to translate $\kappa$ scores into biologically grounded weight updates, preserving device-specific nonlinearities in the potentiation/depression behavior. 

In addition, we considered an \textbf{Ideal (analytical) Spline SADP} kernel, designed as a synthetic reference model. 
Unlike the device-specific spline, the ideal kernel was generated from smooth exponential-like potentiation and depression curves, symmetrically parameterized by decay constants and amplitudes. 
The purpose of this variant was to evaluate SADP performance under an analytically well-behaved update rule that is not constrained by device non-idealities, providing a useful “upper bound” baseline for spline-based dynamics.

We evaluated the SADP framework using a variety of configurations. Models varied in kernel type (\texttt{linear} or \texttt{spline}), temporal encoding (\texttt{rate} or \texttt{TTFS}), and feature dimensionality (either \texttt{1layer} with 400 features or \texttt{1layer small} with 64 features learnt in the output layer of the neural network). All models were trained using unsupervised SADP updates for 10 epochs, and a downstream shallow vanilla neural network classifier was trained for 50 epochs on the extracted features.

\textbf{Training Pipeline.} Each model was trained using a fixed 10 epochs of unsupervised SADP learning with a batch size of 64 to approximate online plasticity. We performed a grid search across all combinations of SADP rule (linear or spline), temporal encoding (rate or TTFS), and network size. Two architectures were considered: a full model learning 400 features through SADP, and a reduced model learning only 64 hidden features with SADP. After SADP training, a standard feedforward classifier with one hidden layer (256 neurons) followed by a softmax output was trained on the extracted spike features using categorical cross-entropy loss for 50 supervised epochs. Evaluation was done on the test set, and the total runtime included both SADP learning and classifier training phases. This experimental protocol was designed to evaluate the computational efficiency of the proposed learning rules without significantly compromising performance.

\textbf{Classical STDP Baseline:} 
As a baseline for the feature extraction process, we implemented a pair-based STDP model using the \texttt{PostPre} rule in BindsNET~\cite{hazan2018bindsnet}, closely following the formulation of Bi and Poo~\cite{bi1998synaptic}. In this update scheme, synapses are potentiated when a presynaptic spike precedes a postsynaptic spike and depressed otherwise, capturing the hallmark temporally asymmetric Hebbian window. The update magnitudes were set to $A_+ = 10^{-4}$ and $A_- = 10^{-4}$, values chosen to balance stability and plasticity across the 10 training epochs. Synaptic weights were initialized uniformly in $[0,0.3]$ and constrained to remain within $[0,1]$.  

The network architecture consisted of 784 input neurons (corresponding to image pixels) fully connected to 400 leaky integrate-and-fire (LIF) neurons with membrane traces enabled. Inputs were converted into 10-timestep Poisson spike trains proportional to pixel intensity, providing a simple but effective rate-based encoding. During each presentation, hidden-layer spikes were accumulated over the 10 timesteps, producing a 400-dimensional feature vector. These features were subsequently passed to a downstream feedforward classifier with structure 400–256–128–10, trained with Adam optimizer on categorical cross-entropy.  

We selected classical STDP as a benchmark because it is biologically plausible, efficient, and event-driven, requiring no global signals or memory traces, which makes it attractive for neuromorphic substrates~\cite{davies2018loihi}. Its simplicity has enabled robust unsupervised learning in digit recognition tasks, particularly when paired with lateral inhibition and threshold adaptation~\cite{diehl2015unsupervised}. Several extensions of the basic model exist: trace-based STDP variants~\cite{morrison2008phenomenological, stimberg2019brian2} accelerate simulation by approximating spike timing with exponentially decaying traces; triplet-based rules~\cite{pfister2006triplets} capture higher-order temporal correlations that improve selectivity; and voltage-based rules~\cite{clopath2010voltage} incorporate postsynaptic depolarization for more biologically realistic dynamics. While we restrict ourselves here to the standard pair-based formulation for comparability, these variants highlight the richness of STDP as a modeling framework.  

\textbf{Hebbian Baseline:} 
As a complementary benchmark, we implemented a rate-based Hebbian learning rule in BindsNET, embodying the principle that ``cells that fire together wire together''~\cite{hebb1949organization}. Here, synaptic weights are strengthened when pre- and postsynaptic neurons are co-active, independent of precise spike timing. In our experiments $\eta = 10^{-3}$ is the learning rate, and $\lambda = 10^{-2}$ is a weight decay coefficient. This Oja-style term~\cite{oja1982simplified} prevents unbounded growth and introduces implicit competition among synapses. Updates were applied online after each sample presentation.  

The network setup paralleled the STDP baseline: 784 input neurons connected to 400 LIF neurons with weights initialized in $[0,0.3]$ and clipped to $[0,1]$, 10-timestep Poisson input coding, and the same downstream classifier.  

Hebbian learning is biologically plausible, local, and computationally lightweight. It has historically been central to models of cortical map formation and early visual development~\cite{linsker1986local, miller1994role}. In contrast to STDP, Hebbian plasticity emphasizes spike coincidence rather than timing, making it naturally better aligned with rate-coded or slowly varying inputs. It also forms the foundation of several unsupervised SNN models~\cite{diehl2015unsupervised}, particularly when combined with normalization or competitive mechanisms.

\textbf{Evaluation Metrics.} Each model configuration—across SADP, STDP and Hebbian—was evaluated using three standard metrics. Classification accuracy on the test set was used to assess predictive performance. The F1 score, computed as the harmonic mean of precision and recall and macro-averaged over all classes, was also included. Finally, runtime was recorded as the total wall-clock time from the beginning of unsupervised training to the end of supervised classifier training.

This benchmarking framework thus enables a systematic and fair comparison of biologically plausible synaptic learning rules, emphasizing both representational quality and computational efficiency. The experiments were designed to assess whether SADP-based learning can significantly reduce training time—particularly in hardware-relevant low-power settings—without substantially compromising performance.

\section{Results}

Figure \ref{fig:mnist_fmnist} illustrates the evolution of weight norms during SADP training and the validation accuracy of downstream classifiers for MNIST and FMNIST datasets. As expected, rate-coded models produced smoother convergence and higher accuracy compared to TTFS-coded models, with weight norms flattening over epochs, indicating empirical convergence of parameters. The comparison across SADP variants, together with STDP and Hebbian baselines, reveals several key trends.

On MNIST, the best performance was achieved by the rate-coded model trained with the ideal (analytical) spline SADP kernel (Fig.~\ref{fig:sadp_kernel_ideal}) on the full 400-feature network, which reached 
91.06\% accuracy with an F1 score of 0.909. This slightly surpassed the \textbf{linear SADP kernel} on the full network (90.68 \%) and clearly outperformed the \textbf{device-derived spline SADP kernel} (88.70 \%). Notably, these results show that SADP can provide both a robust ideal reference (through the synthetic kernel) and a pathway toward hardware-specific learning (through the device-derived kernel).

Among smaller classifiers with only 64 learned features, the rate-coded linear SADP model reached 79.48 \% accuracy, while the ideal spline and device-derived spline variants achieved 76.03 \% and 70.43 \%, respectively. While smaller models reduce per-epoch runtime drastically (137–148 s versus 720–760 s for full models), they come at a notable cost in predictive performance. Under TTFS coding, kernel choice was decisive: the ideal spline SADP kernel maintained strong performance (88.86 \% full, 71.71 \% small), whereas both linear and device-derived spline kernels collapsed (53–54 \% for full models and below 35 \% for small models). This highlights that bounded, smooth ideal kernels confer robustness under temporally sparse encodings where simpler or device-derived kernels struggle.

On FMNIST, the same overall pattern emerged, though accuracies were lower due to dataset complexity. The best-performing model was again the rate-coded full network trained with the ideal spline SADP kernel, which achieved \textbf{78.73 \%} accuracy, closely followed by the linear SADP kernel (78.45 \%). The device-derived spline kernel lagged behind at 72.32 \%. For the smaller networks, accuracies were 70.0 \% (linear), 69.42 \% (ideal spline), and 59.2 \% (device-derived spline). TTFS once again posed challenges: the linear SADP model trained with TTFS coding on the full network reached only 33.63 \%, and the device-derived spline variant dropped further to 21.04 \%. In contrast, the ideal spline SADP with TTFS preserved strong performance (78.26 \% full, 68.58 \% small), nearly matching its rate-coded counterpart. This underlines the stabilizing role of smooth ideal kernels, which preserve learning ability even when spike redundancy is minimal.

Runtime analysis highlights another advantage of SADP. Small classifiers required only 135–150 s per epoch, while full-size models took 720–770 s. STDP baselines, in comparison, were dramatically slower at over 2500 s per epoch, yet only achieved 12.9 \% accuracy on MNIST and 10.7 \% on FMNIST. Hebbian learning was faster (405–457 s/epoch) but failed to extract useful features, plateauing near 11 \% accuracy. Thus, SADP offers a unique balance of efficiency and representational quality: its runtime is only moderately higher than Hebbian, but its accuracy is vastly superior; compared with STDP, it is both faster and more accurate.

Comparing SADP variants reveals distinct strengths. The linear SADP kernel is highly competitive under rate coding, nearly matching the ideal spline kernel and exceeding the device-derived spline kernel, suggesting that simple agreement-driven updates suffice when spike trains are dense and redundant. However, in TTFS settings, the linear kernel collapses, reflecting its inability to cope with sparse timing information. The device-derived spline kernel underperforms relative to both linear and ideal variants, particularly under TTFS, which suggests that raw device nonlinearities can hinder generalization unless carefully regularized. By contrast, the ideal spline kernel consistently achieves the highest or near-highest performance across both coding schemes and datasets, showing that ideally well-behaved update rules provide robustness where other variants fail. Taken together, these results demonstrate that while the linear SADP kernel is efficient and strong under dense encodings, the ideal spline SADP kernel is the most versatile, and the device-derived spline kernel is valuable for aligning learning directly with neuromorphic hardware characteristics.

\begin{figure}[htbp!]
    \centering
    \begin{minipage}[b]{\textwidth}
        \centering
        \includegraphics[width=\textwidth]{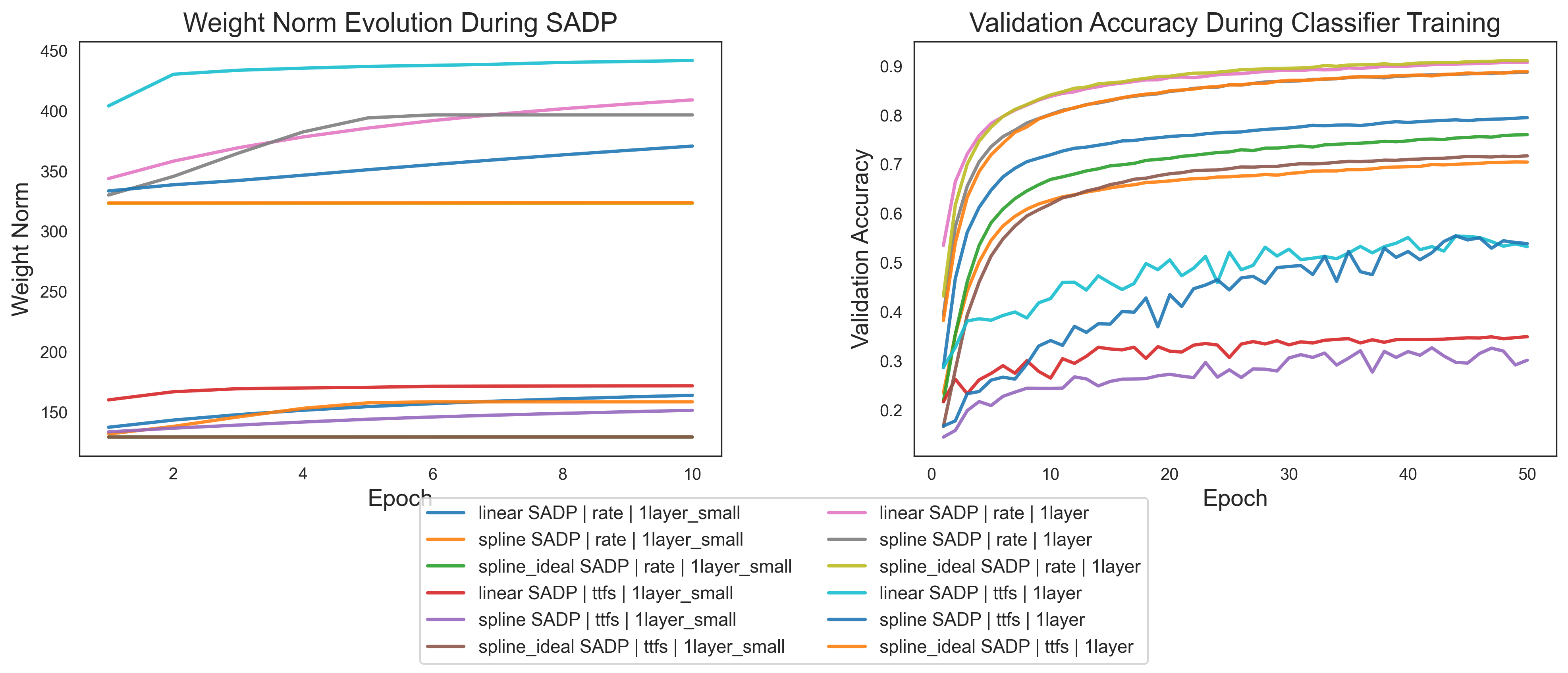}
        \caption*{\textbf{Figure \ref{fig:mnist_fmnist} (Top)}: MNIST results — weight norm (left) over SADP epochs and validation accuracy over classifier training epochs (right).}
    \end{minipage}
    \vspace{1em}

    \begin{minipage}[b]{\textwidth}
        \centering
        \includegraphics[width=\textwidth]{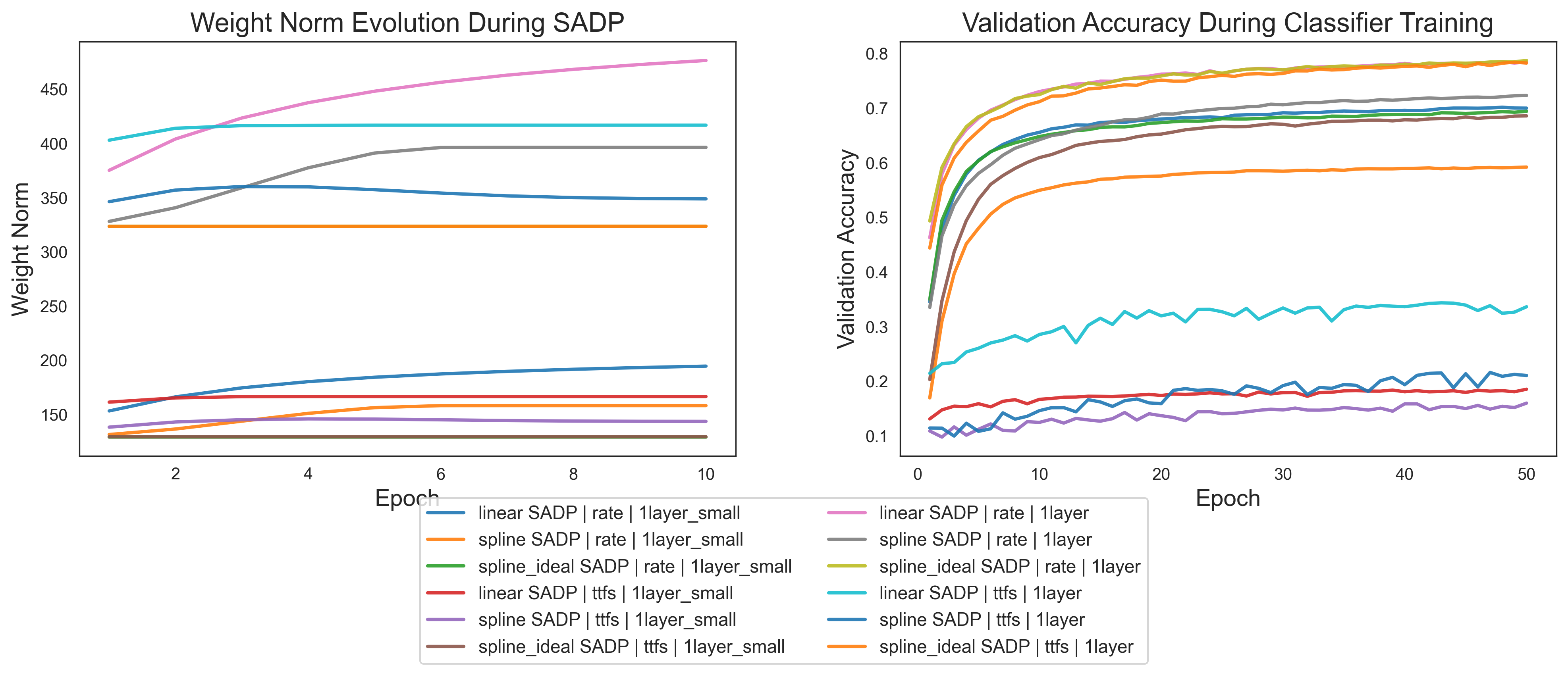}
        \caption*{\textbf{Figure \ref{fig:mnist_fmnist} (Bottom)}: FMNIST results — weight norm (left) over SADP epochs and validation accuracy over classifier training epochs (right).}
    \end{minipage}
    \vspace{2em}
    \caption{\textbf{Training dynamics and classifier performance on MNIST and FMNIST.} Both datasets show the evolution of weight norms during SADP training (10 epochs) and classifier validation accuracy (50 epochs). Models differ by kernel, coding scheme, and classifier size.}
    \label{fig:mnist_fmnist}
\end{figure}

\begin{table}[h]
    \centering
    \caption{MNIST performance summary: accuracy, F1 score, and average runtime per SADP epoch (seconds).}
    \label{tab:mnist}
    \begin{adjustbox}{width=0.95\textwidth}
    \begin{tabular}{lccc}
        \hline
        \textbf{Model Configuration} & \textbf{Validation Accuracy} & \textbf{F1 Score} & \textbf{Runtime / Epoch (s)} \\
        \hline
        \texttt{linear SADP | rate | 1layer}       & 0.9068 & 0.9055 & 720.89 \\
        \texttt{spline SADP | rate | 1layer}       & 0.8870 & 0.8857 & 754.45 \\
        \texttt{spline\_ideal SADP | rate | 1layer} & \textbf{0.9106} & \textbf{0.9094} & 758.83 \\
        \texttt{linear SADP | ttfs | 1layer}       & 0.5325 & 0.4989 & 724.10 \\
        \texttt{spline SADP | ttfs | 1layer}       & 0.5386 & 0.5224 & 750.00 \\
        \texttt{spline\_ideal SADP | ttfs | 1layer} & 0.8886 & 0.8866 & 756.53 \\
        \texttt{linear SADP | rate | 1layer\_small} & 0.7948 & 0.7919 & \textbf{137.32} \\
        \texttt{spline SADP | rate | 1layer\_small} & 0.7043 & 0.6989 & 137.86 \\
        \texttt{spline\_ideal SADP | rate | 1layer\_small} & 0.7603 & 0.7552 & 142.67 \\
        \texttt{linear SADP | ttfs | 1layer\_small} & 0.3493 & 0.3039 & 138.60 \\
        \texttt{spline SADP | ttfs | 1layer\_small} & 0.3014 & 0.2431 & 144.62 \\
        \texttt{spline\_ideal SADP | ttfs | 1layer\_small} & 0.7171 & 0.7112 & 148.63 \\
        \texttt{STDP | rate | 1layer}              & 0.1294 & 0.0518 & 2611.38 \\
        \texttt{Hebbian | rate | 1layer}           & 0.1135 & 0.0204 & 405.93 \\
        \hline
    \end{tabular}
    \end{adjustbox}
\end{table}

\begin{table}[h]
    \centering
    \caption{FMNIST performance summary: accuracy, F1 score, and average runtime per SADP epoch (seconds).}
    \label{tab:fmnist}
    \begin{adjustbox}{width=0.95\textwidth}
    \begin{tabular}{lccc}
        \hline
        \textbf{Model Configuration} & \textbf{Validation Accuracy} & \textbf{F1 Score} & \textbf{Runtime / Epoch (s)} \\
        \hline
        \texttt{linear SADP | rate | 1layer}       & 0.7845 & 0.7825 & 726.83 \\
        \texttt{spline SADP | rate | 1layer}       & 0.7232 & 0.7212 & 756.12 \\
        \texttt{spline\_ideal SADP | rate | 1layer} & \textbf{0.7873} & \textbf{0.7855} & 764.52 \\
        \texttt{linear SADP | ttfs | 1layer}       & 0.3363 & 0.3196 & 724.27 \\
        \texttt{spline SADP | ttfs | 1layer}       & 0.2104 & 0.2052 & 754.57 \\
        \texttt{spline\_ideal SADP | ttfs | 1layer} & 0.7826 & 0.7775 & 766.59 \\
        \texttt{linear SADP | rate | 1layer\_small} & 0.7000 & 0.6982 & \textbf{139.01} \\
        \texttt{spline SADP | rate | 1layer\_small} & 0.5920 & 0.5855 & 144.55 \\
        \texttt{spline\_ideal SADP | rate | 1layer\_small} & 0.6942 & 0.6919 & 146.55 \\
        \texttt{linear SADP | ttfs | 1layer\_small} & 0.1853 & 0.1578 & 139.20 \\
        \texttt{spline SADP | ttfs | 1layer\_small} & 0.1598 & 0.1303 & 144.85 \\
        \texttt{spline\_ideal SADP | ttfs | 1layer\_small} & 0.6858 & 0.6797 & 146.32 \\
        \texttt{STDP | rate | 1layer}              & 0.1069 & 0.0320 & 2525.75 \\
        \texttt{Hebbian | ratw | 1layer}           & 0.1000 & 0.0182 & 404.38 \\
        \hline
    \end{tabular}
    \end{adjustbox}
\end{table}

\vspace{1em}

\begin{remark}[Disclaimer on STDP and Hebbian Baselines]
The results reported here should not be interpreted as evidence that classical STDP or Hebbian learning rules are incapable of learning meaningful features. Substantial prior work has demonstrated that both approaches can achieve high accuracy when given sufficient training time and appropriate architectural support. For example, STDP combined with mechanisms such as winner-take-all inhibition has been used to learn hierarchical and convolutional representations for MNIST, CIFAR-10, and related benchmarks, often exceeding 90\% accuracy and, in some cases, surpassing 95\% \cite{diehl2015unsupervised,Kheradpisheh2018,Ferre2018}. These outcomes typically require extended training (100+ epochs), lateral inhibition, competitive layers, and high spike counts per input (e.g., 350–1000 neurons or timesteps).

In our benchmarks, STDP and Hebbian baselines were intentionally constrained to a lightweight regime of 10 training epochs and 10-timestep spike trains, to enable a fair, resource-matched comparison with SADP. Under these strict constraints, STDP achieved low accuracy and high runtime, while Hebbian learning, though faster, failed to extract sufficiently discriminative features. This performance gap is therefore a consequence of the experimental setting rather than a fundamental limitation of the learning rules.

Within this constrained framework, SADP shows clear advantages: it converges in fewer epochs, maintains competitive runtime efficiency, and learns task-relevant features under both rate and TTFS coding. We attribute these benefits to SADP's global, population-level update mechanism and its ability to incorporate hardware-calibrated, device-specific kernels.
\end{remark}

\section{Discussion} \label{sec:discussion}
\textbf{Effect of Coding and Network Size.}  
The results across MNIST and FMNIST clearly indicate that the choice of temporal coding scheme exerts a stronger influence on performance than the choice of SADP kernel.  

Rate-coded inputs consistently produced higher accuracy and smoother convergence than TTFS-coded inputs (e.g., MNIST: $91.06\%$ vs.\ $88.86\%$ for ideal spline SADP; FMNIST: $78.73\%$ vs.\ $78.26\%$ for ideal spline SADP).  
This shows that denser spike statistics generate more reliable agreement signals for unsupervised feature learning.  

In terms of network size, compact 64-feature models achieved reasonable accuracy while reducing runtime by nearly $5\times$ (137–150 s vs.\ 720–770 s per epoch).  
However, the smaller networks incurred a notable accuracy drop (e.g., MNIST: $91.06\%$ full vs.\ $76.03\%$ small for ideal spline SADP), underscoring a trade-off between efficiency and predictive power that is relevant for resource-constrained neuromorphic systems.

\textbf{Comparative Performance of Kernel Types.}  
A clear hierarchy emerged among SADP kernels. The \textbf{ideal (analytical) spline kernel} consistently delivered the best or near-best results across both datasets and coding schemes (e.g., MNIST: $91.06\%$ rate, $88.86\%$ TTFS; FMNIST: $78.73\%$ rate, $78.26\%$ TTFS).  

The \textit{linear kernel} was highly competitive under rate coding, nearly matching the ideal spline (e.g., MNIST: $90.68\%$; FMNIST: $78.45\%$), but it collapsed under TTFS (MNIST: $53.25\%$; FMNIST: $33.63\%$), reflecting its inability to handle temporally sparse inputs.  

The \textbf{device-derived spline kernel} consistently underperformed both ideal and linear variants (e.g., MNIST: $88.70\%$ rate, $53.86\%$ TTFS; FMNIST: $72.32\%$ rate, $21.04\%$ TTFS). The device-derived spline kernel represents an important step toward hardware-faithful learning.  While its performance currently lags behind the ideal and linear variants (e.g., MNIST: $88.70\%$ rate, $53.86\%$ TTFS; FMNIST: $72.32\%$ rate, $21.04\%$ TTFS), it demonstrates that SADP can directly learn from device-specific kernels measured on memtransistor hardware. The results further suggest a clear principle: the closer a device-derived kernel behaves to the smooth, bounded structure of the ideal spline, the stronger the generalization ability of SADP becomes. This indicates that continued advances in device design and calibration can translate directly into learning performance, making SADP a bridge between idealized theory and practical neuromorphic substrates.

Taken together, these findings highlight that simple linear kernels suffice for dense rate coding, ideal spline kernels provide robustness across both dense and sparse encodings, and device-derived kernels are promising for hardware alignment for general-purpose inference.

\textbf{Comparison with STDP and Hebbian Baselines.}  
Relative to classical rules, SADP showed clear advantages.  
STDP was dramatically slower (over 2500 s per epoch) yet achieved only $12.9\%$ accuracy on MNIST and $10.7\%$ on FMNIST under the constrained 10-epoch regime.  
Hebbian learning was faster (405–457 s/epoch) but plateaued near chance-level performance (11\% accuracy).  
By contrast, SADP reached over $91\%$ MNIST accuracy and nearly $79\%$ FMNIST accuracy within 10 epochs, with runtimes (720–770 s/epoch for full networks; 135–150 s/epoch for small) that were only moderately higher than Hebbian and far below STDP.  
Thus, SADP achieves a favorable balance: faster than STDP, much more accurate than Hebbian, and effective under low-epoch budgets.

\textbf{Limitations and Future Work.}  
These benchmarks intentionally restricted training to 10 epochs and modest feature counts for comparability.  
While this setup highlights SADP’s efficiency, it underestimates the potential of STDP and Hebbian learning, which are known to achieve high performance under longer training and richer architectures.  
SADP itself may also face challenges in deeper networks, where purely local unsupervised updates can dissipate task-relevant structure.  
Future work could combine SADP with top-down or reward-modulated signals (e.g., dopamine-based reinforcement, dendritic prediction, or predictive coding) to enable deeper task-aligned SNNs while preserving efficiency and hardware relevance.

\textbf{Advantages of SADP Over Classical Learning Rules.}  
SADP integrates the efficiency of Hebbian learning and the temporal sensitivity of STDP, while adding robustness and hardware adaptability:

\begin{itemize}[topsep=0pt,itemsep=0pt,parsep=0pt]
    \item \textbf{Fast convergence:} Achieves $>90\%$ MNIST accuracy in only 10 epochs, versus 100+ epochs typically needed by STDP.
    \item \textbf{Efficiency:} $O(T)$ per synapse complexity, compared to $O(S^2)$ spike-pair operations for STDP.
    \item \textbf{Temporal robustness:} Handles spike-train structure without requiring strict pre–post causality.
    \item \textbf{Hardware readiness:} Offers an \textit{ideal spline kernel} for analytical benchmarking and a \textit{device-derived spline kernel} fitted to memtransistor measurements, bridging theory and deployment.
\end{itemize}

Overall, SADP unifies \textit{efficiency}, \textit{temporal structure sensitivity}, and \textit{hardware adaptability}.  
The ideal spline kernel demonstrates broad robustness across coding regimes, the linear kernel provides a simple yet effective solution under dense coding, and the device-derived kernel connects learning directly to neuromorphic substrates, even if its performance currently lags.  
This positions SADP as a practical and versatile rule for next-generation neuromorphic systems.

\section{Conclusion} \label{sec:conclusion}

We introduced SADP, a family of synaptic update rules that use global spike train agreement metrics—such as Cohen's $\kappa$—to drive unsupervised learning in SNNs. Unlike classical STDP, which relies on precise spike-timing coincidences, SADP aggregates information across the entire spike trains. 

Empirical evaluations show that SADP enables competitive downstream classification across a range of kernel types, coding schemes, and network sizes. The choice of temporal coding—especially rate versus TTFS—emerged as a stronger determinant of performance than the kernel type. Compact SADP models with as few as 64 hidden units achieved high classification accuracy, demonstrating the method’s effectiveness for low-resource spiking representation learning. In comparison to both Hebbian and STDP baselines, SADP delivers superior representational quality with competitive runtime, validating its suitability for practical and scalable SNN training.

Importantly, SADP is not limited to idealized kernels: it can learn directly from device-derived kernels measured on neuromorphic hardware. Our results show that while device-specific kernels currently underperform relative to the ideal spline baseline, their effectiveness increases the more closely they approximate the ideal form. This establishes a positive pathway for hardware integration—advances in device calibration and design can translate directly into improved learning performance. Thus, SADP naturally aligns algorithmic progress with hardware development.

The framework’s reliance on binary spike train comparisons also enables hardware-friendly implementations using bitwise logic operations. All SADP variants trained significantly faster than classical STDP, with spline and linear kernels performing comparably under rate coding. These findings highlight SADP’s scalability and efficiency in practical neuromorphic settings.

Nonetheless, the present work focuses on shallow, single-layer networks. Extending SADP to deeper hierarchies remains an open challenge, as global agreement signals may attenuate across layers without supervision or coordination. Future directions include adaptive device-specific kernel learning, reward-modulated SADP variants, and attention-based agreement mechanisms to support deeper, more flexible learning architectures.

Overall, SADP provides a principled, efficient, and hardware-compatible alternative to traditional STDP or Hebbian learning for unsupervised learning in SNNs. By directly supporting both idealized and device-derived kernels, it offers a bridge between theory and neuromorphic hardware. The full source code for SADP, including all spline fitting routines, benchmarking experiments, and comparisons with STDP, is openly available at \href{https://github.com/Saptarshi-Bej/SADP}{GitHub}.

\section*{CRediT authorship contribution statement}
\section*{Author Contributions}
The conceptualization, theoretical formulation, experiments, manuscript writing and supervision were led by the corresponding author, Saptarshi Bej. Gouri Lakshmi and Harshit Kumar contributed to the benchmarking experiments. Muhammed Sahad E and Bikas C. Das were responsible for memtransistor data generation. 

\section*{Conflict of Interest}
The authors have no conflict of interest

\section*{Availability of code and results}
We make the benchmarking results available at \href{https://github.com/Saptarshi-Bej/SADP}{GitHub} along with the environment to facilitate maximum reproducibility.

\section*{Acknowledgment}

\bibliographystyle{elsarticle-num}
\bibliography{main}

\begin{thebibliography}{10}
\expandafter\ifx\csname url\endcsname\relax
  \def\url#1{\texttt{#1}}\fi
\expandafter\ifx\csname urlprefix\endcsname\relax\def\urlprefix{URL }\fi
\expandafter\ifx\csname href\endcsname\relax
  \def\href#1#2{#2} \def\path#1{#1}\fi

\bibitem{bi1998synaptic}
G.-q. Bi, M.-m. Poo, Synaptic modifications in cultured hippocampal neurons: dependence on spike timing, synaptic strength, and postsynaptic cell type, Journal of neuroscience 18~(24) (1998) 10464--10472.
\newblock \href {https://doi.org/10.1523/JNEUROSCI.18-24-10464.1998} {\path{doi:10.1523/JNEUROSCI.18-24-10464.1998}}.

\bibitem{markram1997regulation}
H.~Markram, J.~L{\"u}bke, M.~Frotscher, B.~Sakmann, Regulation of synaptic efficacy by coincidence of postsynaptic aps and epsps, Science 275~(5297) (1997) 213--215.
\newblock \href {https://doi.org/10.1126/science.275.5297.213} {\path{doi:10.1126/science.275.5297.213}}.

\bibitem{bi2001synaptic}
G.-Q. Bi, M.-M. Poo, Synaptic modification by correlated activity: Hebb's postulate revisited, Annual review of neuroscience 24~(1) (2001) 139--166.
\newblock \href {https://doi.org/10.1146/annurev.neuro.24.1.139} {\path{doi:10.1146/annurev.neuro.24.1.139}}.

\bibitem{sjostrom2001rate}
P.~J. Sj{\"o}str{\"o}m, G.~G. Turrigiano, S.~B. Nelson, Rate, timing, and cooperativity jointly determine cortical synaptic plasticity, Neuron 32~(6) (2001) 1149--1164.
\newblock \href {https://doi.org/10.1016/s0896-6273(01)00542-6} {\path{doi:10.1016/s0896-6273(01)00542-6}}.

\bibitem{MasquelierThorpe2007}
T.~Masquelier, S.~J. Thorpe, Unsupervised learning of visual features through spike timing dependent plasticity, PLoS Computational Biology 3~(2) (2007) e31.
\newblock \href {https://doi.org/10.1371/journal.pcbi.0030031} {\path{doi:10.1371/journal.pcbi.0030031}}.

\bibitem{Ferre2018}
P.~Ferr{\'e}, F.~Mamalet, S.~J. Thorpe, Unsupervised feature learning with winner-takes-all based stdp, Frontiers in Computational Neuroscience 12 (2018) 24.
\newblock \href {https://doi.org/10.3389/fncom.2018.00024} {\path{doi:10.3389/fncom.2018.00024}}.

\bibitem{Kheradpisheh2016}
S.~R. Kheradpisheh, M.~Ganjtabesh, T.~Masquelier, Bio-inspired unsupervised learning of visual features leads to robust invariant object recognition, Neurocomputing 205 (2016) 382--392.
\newblock \href {https://doi.org/10.1016/j.neucom.2016.04.029} {\path{doi:10.1016/j.neucom.2016.04.029}}.

\bibitem{Markram2011}
H.~Markram, W.~Gerstner, P.~J. Sj{\"o}str{\"o}m, A history of spike-timing-dependent plasticity, Frontiers in Synaptic Neuroscience 3 (2011) 4.
\newblock \href {https://doi.org/10.3389/fnsyn.2011.00004} {\path{doi:10.3389/fnsyn.2011.00004}}.

\bibitem{Markram2012}
H.~Markram, W.~Gerstner, P.~J. Sj{\"o}str{\"o}m, Spike-timing-dependent plasticity: A comprehensive overview, Frontiers in Synaptic Neuroscience 4 (2012) 2.
\newblock \href {https://doi.org/10.3389/fnsyn.2012.00002} {\path{doi:10.3389/fnsyn.2012.00002}}.

\bibitem{Tian2025}
Y.~Tian, A.~Kembay, N.~D. Truong, J.~K. Eshraghian, O.~Kavehei, Beyond pairwise plasticity: Group-level spike synchrony facilitates efficient learning in spiking neural networks, arXivPublished Apr 14 2025 (2025).

\bibitem{Subthreshold2025}
L.~A. Becker, F.~Baccelli, T.~Taillefumier, Subthreshold variability of neuronal populations driven by synchronous synaptic inputs, bioRxivPreprint (March 2025).
\newblock \href {https://doi.org/10.1101/2025.03.16.643547} {\path{doi:10.1101/2025.03.16.643547}}.

\bibitem{foncelle2018Modulation}
A.~Foncelle, A.~Mendes, et~al., Modulation of spike-timing dependent plasticity: Towards the inclusion of a third factor in computational models, Frontiers in Computational Neuroscience 12 (2018) 49.
\newblock \href {https://doi.org/10.3389/fncom.2018.00049} {\path{doi:10.3389/fncom.2018.00049}}.

\bibitem{clopath2010connectivity}
C.~Clopath, L.~B{\"u}sing, E.~Vasilaki, W.~Gerstner, Connectivity reflects coding: a model of voltage-based stdp with homeostasis, Nature neuroscience 13~(3) (2010) 344--352.
\newblock \href {https://doi.org/10.1038/nn.2479} {\path{doi:10.1038/nn.2479}}.

\bibitem{pfister2006triplets}
J.-P. Pfister, W.~Gerstner, Triplets of spikes in a model of spike timing-dependent plasticity, Journal of Neuroscience 26~(38) (2006) 9673--9682.
\newblock \href {https://doi.org/10.1523/JNEUROSCI.1425-06.2006} {\path{doi:10.1523/JNEUROSCI.1425-06.2006}}.

\bibitem{hebb1949organization}
D.~O. Hebb, The Organization of Behavior, Wiley, 1949.

\bibitem{oja1982simplified}
E.~Oja, A simplified neuron model as a principal component analyzer, Journal of mathematical biology 15~(3) (1982) 267--273.
\newblock \href {https://doi.org/10.1007/BF00275687} {\path{doi:10.1007/BF00275687}}.

\bibitem{bienenstock1982theory}
E.~L. Bienenstock, L.~N. Cooper, P.~W. Munro, Theory for the development of neuron selectivity: orientation specificity and binocular interaction in visual cortex, Journal of Neuroscience 2~(1) (1982) 32--48.
\newblock \href {https://doi.org/10.1523/JNEUROSCI.02-01-00032.1982} {\path{doi:10.1523/JNEUROSCI.02-01-00032.1982}}.

\bibitem{gerstner2014neuronal}
W.~Gerstner, W.~M. Kistler, R.~Naud, L.~Paninski, Neuronal Dynamics: From Single Neurons to Networks and Models of Cognition, Cambridge University Press, 2014.
\newblock \href {https://doi.org/10.1017/CBO9781107447615} {\path{doi:10.1017/CBO9781107447615}}.

\bibitem{Kheradpisheh2018}
S.~R. Kheradpisheh, M.~Ganjtabesh, S.~J. Thorpe, T.~Masquelier, Stdp-based spiking deep convolutional neural networks for object recognition, Neural Networks 99 (2018) 56--67.
\newblock \href {https://doi.org/10.1016/j.neunet.2017.12.005} {\path{doi:10.1016/j.neunet.2017.12.005}}.

\bibitem{Dong2018}
M.~Dong, X.~Huang, B.~Xu, Unsupervised speech recognition through spike-timing-dependent plasticity in a convolutional spiking neural network, PLoS ONE 13~(11) (2018) e0204596.
\newblock \href {https://doi.org/10.1371/journal.pone.0204596} {\path{doi:10.1371/journal.pone.0204596}}.

\bibitem{song2000competitive}
S.~Song, K.~D. Miller, L.~F. Abbott, Competitive hebbian learning through spike-timing-dependent synaptic plasticity, Nature Neuroscience 3~(9) (2000) 919--926.
\newblock \href {https://doi.org/10.1038/78829} {\path{doi:10.1038/78829}}.

\bibitem{yang2025causal}
X.~Yang, B.~Doiron, \href{https://arxiv.org/abs/2501.09296}{Causal spike timing dependent plasticity prevents assembly fusion in recurrent networks}, aRxiv (2025).
\newline\urlprefix\url{https://arxiv.org/abs/2501.09296}

\bibitem{Chawla1999}
D.~Chawla, E.~D. Lumer, K.~J. Friston, The relationship between synchronization among neuronal populations and their mean activity levels, Neural Computation 11~(6) (1999) 1389--1411.
\newblock \href {https://doi.org/10.1162/089976699300016287} {\path{doi:10.1162/089976699300016287}}.

\bibitem{Nessler2013}
B.~Nessler, M.~Pfeiffer, L.~Buesing, W.~Maass, Bayesian computation emerges in generic cortical microcircuits through spike-timing-dependent plasticity, PLoS Computational Biology 9~(4) (2013) e1003037.
\newblock \href {https://doi.org/10.1371/journal.pcbi.1003037} {\path{doi:10.1371/journal.pcbi.1003037}}.

\bibitem{Habenschuss2013}
S.~Habenschuss, H.~Puhr, W.~Maass, Emergence of optimal decoding of population codes through stdp, Neural Computation 25~(6) (2013) 1371--1407.
\newblock \href {https://doi.org/10.1162/NECO_a_00446} {\path{doi:10.1162/NECO_a_00446}}.

\bibitem{sagar2022emulation}
S.~Sagar, K.~Udaya~Mohanan, S.~Cho, et~al., \href{https://doi.org/10.1038/s41598-022-07505-9}{Emulation of synaptic functions with low voltage organic memtransistor for hardware oriented neuromorphic computing}, Scientific Reports 12~(1) (2022) 3808.
\newblock \href {https://doi.org/10.1038/s41598-022-07505-9} {\path{doi:10.1038/s41598-022-07505-9}}.
\newline\urlprefix\url{https://doi.org/10.1038/s41598-022-07505-9}

\bibitem{lee2013effective}
H.~J. Lee, J.~H. Hwang, K.~B. Choi, S.-G. Jung, K.~N. Kim, Y.~S. Shim, C.~H. Park, Y.~W. Park, B.-K. Ju, \href{https://doi.org/10.1021/am403044r}{Effective indium-doped zinc oxide buffer layer on silver nanowires for electrically highly stable, flexible, transparent, and conductive composite electrodes}, ACS Applied Materials \& Interfaces 5~(21) (2013) 10397--10403.
\newblock \href {https://doi.org/10.1021/am403044r} {\path{doi:10.1021/am403044r}}.
\newline\urlprefix\url{https://doi.org/10.1021/am403044r}

\bibitem{mukherjee2021superionic}
A.~Mukherjee, et~al., \href{https://doi.org/10.1063/5.0069478}{Superionic rubidium silver iodide gated low voltage synaptic transistor}, Applied Physics Letters 119~(25) (2021) 253502.
\newblock \href {https://doi.org/10.1063/5.0069478} {\path{doi:10.1063/5.0069478}}.
\newline\urlprefix\url{https://doi.org/10.1063/5.0069478}

\bibitem{sagar2019unconventional}
S.~Sagar, A.~Dey, B.~C. Das, \href{https://doi.org/10.1021/acsaelm.9b00522}{Unconventional redox-active gate dielectrics to fabricate high performance organic thin-film transistors}, ACS Applied Electronic Materials 1~(11) (2019) 2314--2324.
\newblock \href {https://doi.org/10.1021/acsaelm.9b00522} {\path{doi:10.1021/acsaelm.9b00522}}.
\newline\urlprefix\url{https://doi.org/10.1021/acsaelm.9b00522}

\bibitem{vandeBurgt2018organic}
Y.~van~de Burgt, A.~Melianas, S.~T. Keene, et~al., \href{https://doi.org/10.1038/s41928-018-0103-3}{Organic electronics for neuromorphic computing}, Nature Electronics 1~(7) (2018) 386--397.
\newblock \href {https://doi.org/10.1038/s41928-018-0103-3} {\path{doi:10.1038/s41928-018-0103-3}}.
\newline\urlprefix\url{https://doi.org/10.1038/s41928-018-0103-3}

\bibitem{hazan2018bindsnet}
H.~Hazan, D.~J. Saunders, H.~Khan, D.~Patel, D.~T. Sanghavi, H.~T. Siegelmann, R.~Kozma, Bindsnet: A machine learning-oriented spiking neural networks library in python, Frontiers in Neuroinformatics 12 (2018) 89.
\newblock \href {https://doi.org/10.3389/fninf.2018.00089} {\path{doi:10.3389/fninf.2018.00089}}.

\bibitem{davies2018loihi}
M.~Davies, N.~Srinivasa, T.-M. Lin, G.~Chinya, Y.~Cao, S.~Choday, G.~Deal, P.~Dimoulas, B.~Felt, A.~Fiebig, et~al., Loihi: A neuromorphic manycore processor with on-chip learning, IEEE Micro 38~(1) (2018) 82--99.
\newblock \href {https://doi.org/10.1109/MM.2018.112130359} {\path{doi:10.1109/MM.2018.112130359}}.

\bibitem{diehl2015unsupervised}
P.~U. Diehl, M.~Cook, Unsupervised learning of digit recognition using spike-timing-dependent plasticity, Frontiers in Computational Neuroscience 9 (2015) 99.
\newblock \href {https://doi.org/10.3389/fncom.2015.00099} {\path{doi:10.3389/fncom.2015.00099}}.

\bibitem{morrison2008phenomenological}
A.~Morrison, M.~Diesmann, W.~Gerstner, Phenomenological models of synaptic plasticity based on spike timing, Biological Cybernetics 98~(6) (2008) 459--478.
\newblock \href {https://doi.org/10.1007/s00422-008-0233-1} {\path{doi:10.1007/s00422-008-0233-1}}.

\bibitem{stimberg2019brian2}
M.~Stimberg, R.~Brette, D.~F. Goodman, Brian 2, an intuitive and efficient neural simulator, eLife 8 (2019) e47314.
\newblock \href {https://doi.org/10.7554/eLife.47314} {\path{doi:10.7554/eLife.47314}}.

\bibitem{clopath2010voltage}
C.~Clopath, W.~Gerstner, Voltage and spike timing interact in stdp – a unified model, Frontiers in Synaptic Neuroscience 2 (2010) 25.
\newblock \href {https://doi.org/10.3389/fnsyn.2010.00025} {\path{doi:10.3389/fnsyn.2010.00025}}.

\bibitem{linsker1986local}
R.~Linsker, From basic network principles to neural architecture: Emergence of orientation columns, Proceedings of the National Academy of Sciences 83~(21) (1986) 8779--8783.
\newblock \href {https://doi.org/10.1073/pnas.83.22.8779} {\path{doi:10.1073/pnas.83.22.8779}}.

\bibitem{miller1994role}
K.~D. Miller, D.~J. MacKay, The role of constraints in hebbian learning, Neural Computation 6~(1) (1994) 100--126.
\newblock \href {https://doi.org/10.1162/neco.1994.6.1.100} {\path{doi:10.1162/neco.1994.6.1.100}}.

\end{thebibliography}

\section*{Appendix A: Computational Efficiency of SADP}

Synaptic learning rules in spiking neural networks (SNNs) often involve significant computational cost, especially when tracking precise spike timing relationships. In this appendix, we compare the time complexity of classical pairwise STDP with that of SADP. We show that SADP offers substantial computational savings, particularly in high-firing regimes or neuromorphic hardware implementations.

\begin{theorem}[Computational Efficiency of SADP vs.\ STDP]
Consider an SNN with \( N_{\mathrm{pre}} \) pre-synaptic neurons and \( N_{\mathrm{post}} \) post-synaptic neurons, observed over \( T \) discrete time steps. Assume each pre- and post-synaptic neuron emits at most \( S \) spikes within this window. Then:

\begin{enumerate}
    \item The time complexity of classical pairwise STDP (in its direct implementation) is
    \[
    \mathcal{O}(N_{\mathrm{pre}} \cdot N_{\mathrm{post}} \cdot S^2),
    \]
    due to the need to compare all \( S \times S \) spike time pairs between each neuron pair.

    \item The time complexity of SADP, which computes an aggregate agreement statistic (e.g., Cohen's \( \kappa \)) over binary spike trains, is
    \[
    \mathcal{O}(N_{\mathrm{pre}} \cdot N_{\mathrm{post}} \cdot T),
    \]
    since spike trains are compared once per time step to compute agreement metrics.
\end{enumerate}

Hence, if \( S \gg \sqrt{T} \), SADP achieves strictly lower computational cost than pairwise STDP. Moreover, because agreement metrics operate on binary vectors, SADP admits highly optimized bitwise implementations suitable for parallel and neuromorphic hardware.
\end{theorem}

\begin{proof}
Let \( s^p_i(t), s^q_j(t) \in \{0,1\} \) denote the spike trains of pre-synaptic neuron \( i \) and post-synaptic neuron \( j \), respectively, over \( T \) time steps. Assume:
\[
\sum_{t=1}^T s^p_i(t) \leq S, \quad \sum_{t=1}^T s^q_j(t) \leq S.
\]

\textbf{(1) Classical STDP.}  
For each synapse \( (i, j) \), pairwise STDP compares all pairs of spike times between \( s^p_i(t) \) and \( s^q_j(t) \). Since each train has at most \( S \) spikes, this results in \( \mathcal{O}(S^2) \) comparisons per synapse. Across all \( N_{\mathrm{pre}} \times N_{\mathrm{post}} \) synapses, the total complexity is:
\[
\mathcal{O}(N_{\mathrm{pre}} \cdot N_{\mathrm{post}} \cdot S^2).
\]

\textbf{(2) SADP.}  
Instead of evaluating individual spike pairs, SADP computes an aggregate agreement metric (e.g., Cohen’s \( \kappa \)) over the entire spike train. This requires a single pass over \( T \) time steps per synapse, with bitwise comparisons and a fixed number of scalar operations (e.g., summing agreement counts and applying the \( \kappa \) formula). Thus, per synapse:
\[
\mathcal{O}(T) \quad \Rightarrow \quad \mathcal{O}(N_{\mathrm{pre}} \cdot N_{\mathrm{post}} \cdot T) \text{ overall}.
\]

\textbf{(3) Comparison.}  
SADP is more efficient when:
\[
T < S^2 \quad \Rightarrow \quad S > \sqrt{T}.
\]
This condition is commonly satisfied in high-firing regimes, where neurons spike multiple times in a short observation window.

\textbf{(4) Hardware Advantage.}  
Because spike trains are binary, agreement metrics like \( \kappa \) can be computed via fast bitwise operations (e.g., XOR, AND, popcount), allowing efficient parallelization on CPUs, GPUs, or neuromorphic cores. In contrast, STDP requires dynamic scheduling of spike pair comparisons, which is more memory- and compute-intensive.
\end{proof}

\begin{remark}
While optimized STDP implementations may restrict comparisons to a fixed temporal window (e.g., \(\Delta t \leq \tau\)), the worst-case quadratic bound still holds when spike counts are dense. SADP avoids this entirely by summarizing correlation via agreement, not explicit timing.
\end{remark}

\begin{remark}
SADP’s complexity scales linearly with the number of time steps, not the number of spikes. This makes it particularly effective for rate-coded or high-firing regimes, where traditional STDP would incur quadratic costs.
\end{remark}

\section*{Appendix B: Device Fabrication and Electrical Characterization}

Side-gated iontronic organic memtransistors were fabricated on 15~mm $\times$ 20~mm indium tin oxide (ITO)-coated glass substrates (sheet resistance: 4~$\Omega$/sq). Substrates were sequentially cleaned in deionized water, acetone, ethanol, and Milli-Q water under ultrasonic agitation, followed by hot-air drying. 

Surface hydrophobicity was induced by spin-coating hexamethyldisilazane (HMDS) at 4500~rpm for 60~s, followed by annealing at 96$^\circ$C for 60~s. A positive photoresist (S1813) was spin-coated at 4500~rpm for 60~s and soft-baked at 110$^\circ$C for 1~min. The source (S), drain (D), and gate (G) electrodes were lithographically patterned using a maskless photolithography tool (DMO MicroWriter), developed in 0.26~M TMAOH for 17~s, rinsed, and dried. ITO was subsequently etched in 9~M HCl for 16~min. Optical inspection and conductivity testing confirmed complete and defect-free electrode definition.

A second lithography step defined the channel window (length: 10~$\mu$m, width: 300~$\mu$m, gate–channel spacing: 100~$\mu$m). Following HMDS treatment, poly(3-hexylthiophene) (P3HT) was spin-coated at 1500~rpm for 60~s, annealed at 120$^\circ$C for 20~min, and cooled to room temperature. Finally, a solid redox electrolyte composed of polyethylene oxide and ethyl viologen diperchlorate (PEO:EV(ClO$_4$)$_2$) was drop-cast onto the gate–channel junction and vacuum-dried to crystallize the film, completing the device structure shown in Fig.~\ref{fig:oect_characterization}a.

\subsection*{Electrical Characterization}

Prior to electrolyte deposition, two-terminal source–drain sweeps ($V_{\mathrm{SD}} \in [-1,1]$~V) across the patterned P3HT channel exhibited linear, hysteresis-free conduction (Fig.~\ref{fig:oect_characterization}c), confirming its intrinsic semiconducting behavior. Following electrolyte integration, transfer characteristics ($V_{\mathrm{G}} \in [+3, -3]$~V, $V_{\mathrm{D}}=-0.5$~V) revealed a reproducible hysteresis loop (Fig.~\ref{fig:oect_characterization}d), consistent with enhancement-mode p-channel operation. The suppressed gate current ($I_{\mathrm{G}}$) originates from redox coupling between the P3HT channel and the EV(ClO$_4$)$_2$ electrolyte, highlighting stable ionic gating with minimal leakage.

\subsection*{Synaptic Functionality}

To probe neuromorphic applicability, excitatory (potentiating) and inhibitory (depressing) voltage pulse trains were applied at the gate. Specifically, 1000 consecutive excitatory pulses ($V_{\mathrm{G}}=-3.0$~V, 50~ms) and inhibitory pulses ($V_{\mathrm{G}}=+1.0$~V, 50~ms) were delivered, with a read operation ($V_{\mathrm{D}}=-0.5$~V, 50~ms) after each pulse. The resulting conductance modulation (Fig.~\ref{fig:memtransistor_data}) shows gradual, symmetric, and nearly linear potentiation/depression, a desirable feature for analog weight updates in neuromorphic computing.

\subsection*{Kernel Extraction for SADP}

The experimentally measured potentiation and depression trajectories were fitted with smooth spline functions to yield bounded, device-specific update kernels. These kernels were directly integrated into the Spike Agreement–Dependent Plasticity (SADP) framework used in the main text, ensuring that simulated learning dynamics accurately reflect the measured conductance evolution of the fabricated devices.

\end{document}